\documentclass[letterpaper, 10 pt, conference]{ieeeconf}

\IEEEoverridecommandlockouts
\usepackage{hyperref}
\usepackage{amssymb,amsfonts}
\usepackage{algorithmic}
\usepackage{graphicx}
\usepackage{textcomp}
\usepackage{xcolor}
\usepackage{amsthm}
\usepackage{amsmath}
\usepackage{bm}
\usepackage{dsfont}
\usepackage{subfigure}
\usepackage{subcaption}
\usepackage[algo2e,ruled,noend]{algorithm2e}

\def\BibTeX{{\rm B\kern-.05em{\sc i\kern-.025em b}\kern-.08em
    T\kern-.1667em\lower.7ex\hbox{E}\kern-.125emX}}
\begin{document}
% Remember to add package
% \usepackage{amsthm} 

\newcommand{\M}{\bm M}
\newcommand{\K}{\mathcal K}
\newcommand{\AK}{A_{\Kb}}
\newcommand{\Kb}{\mathbb K}
\newcommand{\E}{\mathbb E}
\newcommand{\Ec}{\mathcal E}
\newcommand{\N}{\mathbb N}
\newcommand{\A}{\mathcal A}
\newcommand{\F}{\mathcal F}
\newcommand{\Ss}{\mathcal S}
\newcommand{\I}{\mathcal I}
\newcommand{\X}{\mathcal X}
\newcommand{\U}{\mathcal U}
\newcommand{\W}{\mathcal W}
\newcommand{\Pb}{\mathbb P}
\newcommand{\B}{\mathbb B}
\newcommand{\Sb}{\mathbb S}
\newcommand{\D}{\mathcal D}
\newcommand{\R}{\mathbb R}
\newcommand{\ball}{\bar{\mathcal{B}} }
\newcommand{\norm}[1]{\left\lVert#1\right\rVert}
\newcommand{\one}{\mathds{1}}
\newcommand{\bd}[1]{\textbf{#1}}

\newcommand{\red}[1]{\textcolor{red}{#1}}
\newcommand{\blue}[1]{\textcolor{blue}{#1}}

\newtheorem{theorem}{Theorem}
\newtheorem{conjecture}{Conjecture}
\newtheorem*{example}{Example}

\newtheorem{lemma}{Lemma}
\newtheorem{proposition}{Proposition}
\newtheorem{corollary}{Corollary}

\theoremstyle{definition}
\newtheorem{definition}{Definition}
\newtheorem{assumption}{Assumption}
\newtheorem{remark}{Remark}

\title{\LARGE \bf Stochastic Gradient Descent with Strategic Querying}
%\author{Nanfei Jiang}
\author{Nanfei Jiang, Hoi-To Wai, Mahnoosh Alizadeh}

\maketitle

\begin{abstract}
This paper considers a finite-sum optimization problem under first-order queries and investigates the benefits of strategic querying on stochastic gradient-based methods compared to uniform querying strategy. 
We first introduce Oracle Gradient Querying (OGQ), an idealized algorithm that selects one user's gradient yielding the largest possible expected improvement (EI) at each step. However, OGQ assumes oracle access to the gradients of all users to make such a selection, which is impractical in real-world scenarios. To address this limitation, we propose Strategic Gradient Querying (SGQ), a practical algorithm that has better transient-state performance than SGD while making only one query per iteration. For smooth objective functions satisfying the Polyak-Lojasiewicz condition, we show that under the assumption of EI heterogeneity, OGQ enhances transient-state performance and reduces steady-state variance, while SGQ improves transient-state performance over SGD. Our numerical experiments validate our theoretical findings.
\end{abstract}

\section{Introduction}
% \htwai{Are we using the right template for CDC?}
Stochastic gradient-based optimization methods have been widely used in large-scale machine learning and distributed optimization applications, thanks to their strong theoretical guarantees and computational efficiency \cite{bottouOptimizationMethodsLargeScale2018a}, \cite{quHarnessingSmoothnessAccelerate2018},
\cite{ghaderyanFastRowStochasticDecentralized2024}. Their simplicity also makes them easy to implement in a wide range of modern optimization tasks. However, existing approaches often suffer from sampling inefficiencies due to the absence of strategic querying. For example, in the classical stochastic gradient descent framework, data points are selected uniformly at random in each iteration, without considering the potential benefits of querying a more informative user over the others. While this lack of strategy is not a bottleneck when the primary challenge lies in reducing computational costs, it becomes a concern in scenarios where data acquisition is expensive, perhaps due to privacy constraints 
\cite{cortesDifferentialPrivacyControl2016},
\cite{wangDifferentialPrivacyLinear2017},  
\cite{yangFederatedLearningPrivacy2020},
\cite{cummingsOptimalDataAcquisition2023}. Examples include personalized healthcare and optimization of societal-scale systems, where the cost of querying each single user is a bigger bottleneck than model training costs. In such settings, the primary goal is to optimize the objective with the fewest possible gradient queries, even if it incurs additional computational overhead during training. This motivates the need for stochastic gradient-based algorithms with high query efficiency, where query efficiency in this paper is measured by the amount of net objective improvement achieved per query. 

If we conceptually divide the optimization process into two stages—transient state and steady state—this paper focuses on improving the transient-state performance rather than the steady-state performance. This emphasis arises because, in the steady state, the objective value typically becomes stable and barely gets improved. Thus, additional gradient queries yield only negligible improvement. Consequently, the marginal gain per query in the steady-state is very small, and the notion of query efficiency loses its relevance in this regime. Since the goal is to obtain a low-precision solution with as few queries as possible, we focus on designing query strategies that primarily enhance transient-state performance.

Specifically, we consider the following unconstrained finite-sum optimization problem:
\begin{equation}
    \min_{x \in \R^d}  \quad f(x) := \frac{1}{n}\sum_{i=1}^{n} f_i(x),
\label{eq: finite sum}
\end{equation}
where each $f_i$ is a component function, representing the cost of individual user $i$. The standard SGD update rule follows:
\begin{equation} 
\label{eq: iteration scheme}
x_{t+1} = x_t - \alpha_t \nabla f_{i_t}(x_t), \quad t\ge 0,
\end{equation}
where $i_t$ is uniformly sampled from $\{1,\ldots,n\}$. 

There has been a rich body of work that proposes SGD variants to improve its performance. One common approach is variance reduction \cite{driggsAcceleratingVariancereducedStochastic2022}, \cite{johnsonAcceleratingStochasticGradient2013}, \cite{defazioSAGAFastIncremental2014}, which aims to decrease the variance introduced due to the use of gradient estimates instead of the true gradients. While these methods greatly improve the performance of SGD in terms of steady-state performance, few of them consider improving transient-state performance from an information gain perspective. Specifically, while most SGD variants propose successfully modified schemes to achieve faster steady-state convergence, they still rely on the same naive sampling strategy as SGD: uniform sampling, without exploiting the heterogeneity between users. Despite its potential benefits, there is currently no systematic framework for understanding how user heterogeneity can be leveraged to achieve transient acceleration in stochastic gradient-based algorithms. This leads us to ask: can we improve the performance of existing stochastic gradient-based algorithms, especially in the early optimization stage, by equipping them with a more effective querying strategy that utilizes user heterogeneity? Designing such an information gain-based querying strategy is the central focus of this paper. 

We note that our goal is different from merely proposing an SGD variant that competes with one specific state-of-the-art algorithm, but rather on providing a framework that analyzes the acceleration potential of all existing stochastic gradient-based algorithms by equipping them with a ``smarter” querying strategy. We take the first step to answer this question here. Specifically, to isolate the impact of a more effective querying strategy, we investigate how to equip standard SGD with one. We thus ask the following question:

\textit{Can we design a querying strategy to replace the uniform random sampling of $i_t$ in \eqref{eq: iteration scheme}, such that, under the same setting as SGD (i.e., querying and using a single user’s gradient per iteration), the proposed strategy achieves better performance than SGD? }

To tackle the above question, we propose a novel strategic querying framework for finite-sum optimization with first-order queries, which selects users' gradients to query based on their potentials to improve the objective function. We formalize this query strategy via the notion of \emph{expected improvement} (EI) that is derived from the classical descent lemma. This helps to maximize the querying efficiency and thus improving the algorithm performance. Our contribution in this paper can be summarized as follows:
\begin{enumerate}
    \item We first design an upper-bound benchmark algorithm, Oracle Gradient Querying (OGQ), which assumes access to the gradients of all users and selects the user that yields the largest EI in the myopic sense to query at each step. Assuming that each of $f_i$ is smooth and satisfies the Polyak-Lojasiewicz condition, together with an assumption on EI heterogeneity, we prove that OGQ outperforms SGD by both enhancing transient-state performance and reducing steady-state variance.
    \item We then propose Strategic Gradient Querying (SGQ), a practical algorithm designed to directly compete with SGD. SGQ mimics the querying strategy of OGQ with low regret while maintaining a single query per iteration. We prove that, under an extra assumption that bounds the gradients difference between users, SGQ consistently improves transient-state performance compared to SGD, although its effect on steady-state variance may increase depending on problem characteristics.
\end{enumerate}

\noindent \textbf{Related Work:} There are two research areas which are close to our work: the first is Multi-Task Bayesian Optimization (MTBO), which explores strategic user querying but in a fundamentally different problem setting. The second involves SGD variants by querying with importance sampling (IS), although these works often lack an analysis of the acceleration benefits of strategic querying.
 
\textit{Multi-Task Bayesian Optimization:} Bayesian Optimization (BO) \cite{belakariaBayesianOptimizationIterative2023a}, \cite{frazierBayesianOptimization2018} is widely used to optimize expensive black-box functions by selecting queries that maximize acquisition functions. One common approach is to define an Expected Improvement (EI) \cite{amentUnexpectedImprovementsExpected2023} to identify the most informative point to query. In this paper, we borrow the concept of EI, but apply it in a fundamentally different setting. MTBO extends BO to efficiently optimize the average or scalarized performance of multiple related tasks \cite{swerskyMultiTaskBayesianOptimization2013}, \cite{toscano-palmerinBayesianOptimizationExpensive2022}. Though the core idea of MTBO is related to strategic querying, MTBO is very different from our framework as it mainly relies on a Gaussian Process (GP) type of modeling instead of using gradient-based type of approaches.

\textit{Importance Sampling:} There are extensive works that formulate querying strategies within a gradient descent-based
framework using importance sampling, e.g.
\cite{katharopoulosNotAllSamples2018}, \cite{elhanchiAdaptiveImportanceSampling2020}, 
\cite{gowerSGDGeneralAnalysis2019},
\cite{hanchiStochasticReweightedGradient2022}, \cite{chenHeteRSGDTacklingHeterogeneous2023a}. However, these works all primarily focus on reducing steady-state variance, without considering the benefits for transient-state acceleration by querying heterogeneous users. The works most closely related to ours are \cite{hanchiStochasticReweightedGradient2022} and \cite{chenHeteRSGDTacklingHeterogeneous2023a}.
The work by \cite{hanchiStochasticReweightedGradient2022} introduces Stochastic Reweighted Gradient Descent (SRG) based on importance sampling instead of random sampling. In each iteration, SRG computes the optimal sampling distribution by minimizing a surrogate term that approximates the true variance and then selects the query indices directly from this optimal distribution. Note that the main focus of \cite{hanchiStochasticReweightedGradient2022} is to reduce the variance in the near-optimal state and not to find a querying strategy to accelerate transient-state convergence. The work by \cite{chenHeteRSGDTacklingHeterogeneous2023a} explores importance sampling in the context of heterogeneous sampling costs, where the expected cost of sampling each index is different. Their proposed algorithm, HeteRSGD, addresses this challenge by balancing variance reduction with sampling cost considerations. They aim to determine the optimal sampling distribution that minimizes asymptotic sampling costs while maintaining effective variance reduction. Both works focus on steady-state variance reduction through carefully designed importance sampling strategies. In contrast, our algorithm deviates from the variance reduction paradigm and focuses on finding a querying strategy based on improvement at each step rather than variance, ensuring more effective descent in both early and potentially later stages of optimization. \\

\noindent \textbf{Paper Overview:}
This paper is structured as follows: Section \ref{section: OGQ algorithm} and \ref{section: SGQ algorithm} detail the development of our proposed OGQ and SGQ algorithm while also introducing the main theoretical contributions of our work. In Section \ref{section: numerical}, we provide numerical experiments to validate the effectiveness of the OGQ and SGQ algorithms. In Section \ref{section: conclusion}, we draw conclusions and outline our future research directions.
Lastly, in Section \ref{section: justification}, we provide supplementary discussions supporting the results presented in the previous sections.

\subsection{Basic Notations}
In this paper, $\norm{\cdot}_1, \norm{\cdot}_2, \norm{\cdot}_{\infty}$ denote the $L_1, L_2, L_{\infty}$ norm for vectors or matrices. Unless otherwise specified, the $\norm{\cdot}$ refers to the $L_2$ norm. Let $x^*$ denote the unique minimizer of $f(x)$, and let $\inf f$ represent its minimum value. We also use $[[1,n]]$ to denote the set of all integers ranging from $1$ to $n$.

\section{The OGQ Algorithm}
\label{section: OGQ algorithm}

We first introduce the Oracle Gradient Querying (OGQ) algorithm, which will serve as our upper-bound benchmark algorithm. OGQ provides an optimized querying strategy in the SGD framework by making the unrealistic assumption of having oracle access to the gradients of all users at each step before choosing whose gradient to query for the purpose of optimization. Specifically, OGQ operates by comparing all users' so-called \textit{Expected Improvement}, and picking a single user gradient to query that provides the biggest improvement to the objective at each step. We will also show that OGQ provably outperforms SGD in both transient-state decaying speed and steady-state variance under our assumptions. 

We make the following standard assumptions on objective function $f(x)$ and each user $f_i(x)$:
\begin{assumption}
     \label{assumption: f_i convex}
Each $f_i: \R^d \rightarrow \R$ is differentiable and convex:
    $$ f_i(y) \ge f_i(x) + \langle \nabla f_i(x), y-x \rangle, \quad \forall x,y\in \R^d.$$
\end{assumption}

\begin{assumption}
    \label{assumption: f_i smooth}
    Each $f_i: \R^d \rightarrow \R$ is $L_i$-smooth :
    $$ \norm{\nabla f_i(x) - \nabla f_i(y)} \le L_i \norm{y - x}, \quad \forall x,y\in \R^d.$$

    By the triangle inequality, it follows that $ f(x) $ is also $ L $-smooth, where $ L $ is given by $L := \frac{1}{n} \sum_{i=1}^n L_i$. Additionally, we define $ L_{\max} := \max_i L_i $.
    
\end{assumption}

\begin{assumption}
\label{assumption: f PL-condition}   
The function $f : \R^d \rightarrow \R$ is differentiable and $\mu$-Polyak-Lojasiewicz, i.e.
$$f(x) - \inf f \le \frac{1}{2\mu}\norm{\nabla f(x)}^2, \quad \forall x \in \R^d. $$
\end{assumption}

Given Assumption \ref{assumption: f_i convex}-\ref{assumption: f PL-condition}, we seek to answer the following key question:
\textit{If at each time step, one oracle has access to gradients from all $f_i(x_t)$ but can only select one user to perform a descent step, which user should be queried to potentially achieve better performance?}

To answer the above question, we limit our attention to myopic measures, and investigate the per-step improvement from $f(x_t)$ to $f(x_{t+1})$. By the $L$-smoothness of $f(x)$, we have 
$$
\begin{aligned}
f(x_{t}) - f(&x_{t+1}) \ge \langle\nabla f(x_t), x_t - x_{t+1}\rangle - \frac{L}{2}\norm{x_t - x_{t+1}}^2 \\
&= \alpha_t \langle\nabla f(x_t), \nabla f_{i_t}(x_t)\rangle - \frac{\alpha_t^2 L}{2}\norm{\nabla f_{i_t}(x_t)}^2.
\end{aligned}
$$

This expression shows that the lower bound on improvement depends on the choice of  $i_t$, as it directly influences  $\nabla f_{i_t}(x_t)$. Thus, to maximize the myopic improvement, we should select the index  $i_t$  that maximizes this bound. To formalize this idea, we define the Expected Improvement (EI) for each user $i\in [[1,n]]$ as:
\begin{equation}
    \text{EI}_i(x_t) := \alpha_t \langle\nabla f(x_t), \nabla f_i(x_t)\rangle - \frac{\alpha_t^2 L}{2}\norm{\nabla f_i(x_t)}^2.
\end{equation}

Then the per-step improvement is lower bounded by:
$$
f(x_{t}) - f(x_{t+1}) \ge \max_i \text{EI}_i(x_t).
$$

The above idea is at the heart of the OGQ algorithm, which aims to  maximize the $\text{EI}_i$ at each time step and choose an informative query for gradient descent, given oracle access to all user gradients.

\begin{algorithm2e}
	\caption{OGQ Algorithm}
	\label{alg: ideal}
	\SetAlgoNoLine
	\DontPrintSemicolon
	\LinesNumbered
	\KwIn{Stepsize $\alpha_t$, initial value $x_0$, Lipschitz constant $L_i$, $i \in [[1,n]]$}
 
	\For{$t=0, 1,2,\ldots,T-1$}{ 
        Query all users' gradient $\nabla f_i (x_t)$\;
        Compute $i_t = \text{argmax}_i \text{ } \alpha_t \langle \nabla f(x_t), \nabla f_i(x_t) \rangle - \frac{\alpha_t^2 L}{2} \norm{\nabla f_i (x_t)}^2$ \;
        Iterate $x_{t+1} = x_t - \alpha_t \nabla f_{i_t}(x_t)$ \;}
        \KwOut{$x_T$}
\end{algorithm2e}

\subsection{Analysis for OGQ}

To analyze the performance of the OGQ algorithm and characterize how much improvement we may obtain from maximizing $\text{EI}_i$ and picking indices strategically, we introduce the following assumption:

\begin{assumption}[\text{EI} Heterogeneity]
    \label{assumption: Hete}
    There exist $C_1, C_2 > 0$, such that for any $x$, and $0 < \alpha \le \frac{1}{2L}$, we have

    $$ 
    \begin{aligned}
    \mathrm{Var}\Big[ & \left\{ \alpha \langle \nabla f(x), \nabla f_i(x)\rangle - \frac{\alpha^2L}{2} \norm{\nabla f_i(x)}^2 \right\}_{i=1}^n \Big] \\
    & \ge C_1 \alpha^2 \norm{\nabla f(x)}^4 + C_2 \alpha^4 L^2 \left( \frac{1}{n} \sum_{i=1}^n \norm{\nabla f_i(x)}^2 \right)^2,
    \end{aligned}
    $$
    % where the expectation and variance is taken with respect to index $i$ under a discrete uniform distribution, where $i \in [[1,n]]$. 
    where ${\rm Var}(\cdot)$ computes the empirical variance over a set of $n$ values, i.e., ${\rm Var}( \{ y_i \}_{i=1}^n ) := \frac{1}{n} \sum_{i=1}^n ( y_i - \frac{1}{n} \sum_{j=1}^n y_j )^2$.
\end{assumption}

Assumption~\ref{assumption: Hete} characterizes the heterogeneity of $\text{EI}_i(x)$. Since  classical SGD randomly selects indices and obtains $\frac{1}{n} \sum_n \text{EI}_i(x_t)$ as per-step improvement, this assumption directly quantifies the potential benefit of maximizing $\text{EI}_i(x)$ instead of taking the average. If the variance of EI is large, then the difference $\max_i \text{EI}_i(x) - \frac{1}{n} \sum_{i=1}^n \text{EI}_i(x)$ will also be large. Intuitively, Assumption~\ref{assumption: Hete} requires the existence of some indices $i$ such that their gradient directions $\nabla f_i$ align well with the full gradient $\nabla f$ (i.e., $\langle \nabla f, \nabla f_i \rangle$ is large) to provide faster convergence, while their norms $\norm{\nabla f_i}$ remain small to help reduce variance. Assumption~\ref{assumption: Hete} is satisfied by many common component functions, including quadratic and logarithmic functions, as long as heterogeneous hyperparameters are present. In fact, Assumption \ref{assumption: Hete} inherently assumes that no point $x$ leads $\nabla f_i (x)$ to be identical for all $i$. Otherwise, since there is no difference between users, maximizing $\text{EI}_i$ would offer no advantage over SGD in this setting, as selecting one user over another would make no difference. 

We leave further discussion and justification of 
Assumption \ref{assumption: Hete} to Section \ref{section: justification}, where we specify the forms of  $C_1$  and  $C_2$  and outline the procedure for determining them.

Before giving the convergence result for the OGQ algorithm, we establish the following lemma that relates the variance of $\text{EI}$ directly to the gain of strategic querying, i.e. $\max_i\text{EI}_i(x) - \frac{1}{n} \sum_{i=1}^n \text{EI}_i(x)$.

\begin{lemma}
\label{lemma: EI difference}
We define the notation $$\tilde{c}(x) = \frac{ \frac{1}{n} \sum_{i=1}^n \text{EI}_i(x)  - \min_i \text{EI}_i(x) }{\max_i \text{EI}_i(x) - \frac{1}{n} \sum_{i=1}^n \text{EI}_i(x) }.$$
Then, the difference between $\max_i \text{EI}_i(x)$ and the average of $\text{EI}_i(x)$ is lower bounded as:
\begin{equation}
\label{eq: max_EI - E_EI}
\max_i \text{EI}_i(x) - \frac{1}{n} \sum_{i=1}^n \text{EI}_i(x) \ge \sqrt{\frac{\mathrm{Var}[ \{ \text{EI}_i(x) \}_{i=1}^n ]}{c}}, \quad
\end{equation}
for any $x \in \mathbb{R}^d$, where $c := \max_{x} \tilde{c}(x) \le n-1 < \infty$, and the definition of ${\rm Var}(\cdot)$ was given in Assumption~\ref{assumption: Hete}.
\end{lemma}

\begin{proof}  
    Let $ M(x) = \max_i \text{EI}_i(x) $, $ m(x) = \min_i \text{EI}_i(x)$, and also let $A(x):= \frac{1}{n} \sum_n \text{EI}_i(x)$ be the empirical average of $\text{EI}_i(x)$. Then, by applying Popoviciu's inequality on variance, we obtain that
    \[
    \mathrm{Var}[ \{ \text{EI}_i(x) \}_{i=1}^n ] \le (M(x) - A(x)) (A(x) - m(x)).
    \]  
    Invoking the definition of $\tilde{c}(x)$, we have
    $$
    A(x) - m(x) \le \tilde{c}(x) \cdot (M(x) - A(x)).
    $$  
    
    Combining the above two inequalities yields  
    \[
    \mathrm{Var}[ \{ \text{EI}_i(x) \}_{i=1}^n ]^2 \le \tilde{c}(x) \cdot (M(x) - A(x))^2,
    \]  
    which establishes the desired result in \eqref{eq: max_EI - E_EI}. We then prove the argument that $c = \max_x \tilde{c}(x) \le n - 1 < \infty$. This can be proved by contradiction: we suppose that there exist $\bar{x}$ such that $$A(\bar{x}) - m(\bar{x}) > (n-1) (M(\bar{x}) - A(\bar{x})).$$
    
    Then by organizing the term, we obtain that $n \cdot A(\bar{x})$ strictly bigger that $m(\bar{x}) + (n-1)M(\bar{x})$, which leads to
    $$ M(\bar{x}) < \frac{ \sum_i \text{EI}_i (\bar{x})  - m(\bar{x}) }{n-1} \le M(\bar{x}). $$
        
The last inequality is due to the fact that $\frac{ \sum_i \text{EI}_i (\bar{x})  - m(\bar{x}) }{n-1} $ is the average of $\text{EI}_i(\bar{x})$ without the minimal element $m(\bar{x})$. As a result, it will be less or equal to the maximal element $M(\bar{x})$, which leads to a contradiction.
\end{proof}

We note that even though the value of $\tilde{c}(x)$ does reach $n-1$ in some extreme cases,  $\tilde{c}(x)$ won't usually grow as $n$ increases. This is because $ f_i(x)$'s are usually not chosen adversarially in practical scenarios. A more detailed discussion on when we can obtain a tighter bound on $\tilde{c}(x)$ can be found in Section \ref{section: justification}.

With this foundation in place, we are now ready to present the analysis of the OGQ algorithm. Since $\max_i \text{EI}_i(x) $ measures how much benefit OGQ obtains in each step, the proof is basically done by separately analyzing $\frac{1}{n} \sum_n \text{EI}_i(x)$ and $\max_i \text{EI}_i(x) - \frac{1}{n} \sum_n \text{EI}_i(x)$, The former follows the standard convergence analysis of SGD, while the latter is done by applying Assumption \ref{assumption: Hete} and Lemma \ref{lemma: EI difference}.
\begin{theorem}[OGQ Algorithm]
    \label{thm: OGQ}
    Let Assumption \ref{assumption: f_i convex}-\ref{assumption: Hete} hold, and define the optimality gap by $G_t := f(x_t) - \inf f$. If a constant stepsize is applied $\alpha_t \equiv \alpha$ and satisfies $0 < \alpha \le \frac{\mu}{2LL_{\max}}$, then our OGQ algorithm gives\footnote{It can be shown that the parameter $c$ and $C_2$ in \eqref{eq: Ideal_bound} satisfies $c\ge 4 C_2$, and thus the term $\sqrt{c/2 } - \sqrt{C_2} \ge (\sqrt{2} - 1)\sqrt{C_2}$ remains positive. The detailed proof can be found in the appendix.}:
    \begin{equation} 
    \begin{aligned}
    \label{eq: Ideal_bound}
        G_t \le \Bigg[ 1 - &\left( 1 + \frac{\sqrt{2}(\sqrt{C_1} + \sqrt{C_2})}{\sqrt{c}} \right) \alpha \mu \Bigg]^t \cdot G_0 \\
        &+ \frac{\alpha L}{\mu} \cdot \frac{\sqrt{c/2 } - \sqrt{C_2}}{\sqrt{c/2 } + (\sqrt{C_1} + \sqrt{C_2})}\cdot \sigma_f^*,
    \end{aligned}
    \end{equation}
where $\sigma_f^* = \mathrm{Var}\left[ \{\nabla f_i(x^*)\}_{i=1}^n \right]$ is the gradient noise at optimal point $x^*$.
\end{theorem}

As a comparison, under the same assumptions and step size requirement, the SGD algorithm gives:
\begin{equation}
    \label{SGD_bound}
    \E[G_t] \leq (1 - \alpha \mu)^t \cdot G_0 + \frac{\alpha L}{\mu} \sigma_f^*.
\end{equation}
From the result in Theorem \ref{thm: OGQ}, we observe that the OGQ algorithm enhances transient-state performance by introducing an additional term, $\sqrt{2}(\sqrt{C_1} + \sqrt{C_2})/\sqrt{c}$, subtracted in the first exponentially decaying term. Meanwhile, the steady-state error is also reduced given the fact that $\frac{\sqrt{c/2 } - \sqrt{C_2}}{\sqrt{c/2 } + (\sqrt{C_1} + \sqrt{C_2})}< 1$. Additionally, we note that since OGQ is a deterministic algorithm, the result showing in \eqref{eq: Ideal_bound} is without randomness, which is distinct from the result of SGD where the bound of optimality gap is given in expectation.

Theorem \ref{thm: OGQ} shows that the OGQ algorithm provably outperforms SGD by leveraging strategic querying, highlighting the potential benefits of an effective querying strategy.

\section{The SGQ Algorithm}
\label{section: SGQ algorithm}

The OGQ algorithm requires oracle access to all $f_i$'s gradients to determine its querying strategy, and is only meant to serve as a baseline. In this section, we introduce the SGQ algorithm. The SGQ algorithm mimics the query strategy of OGQ by utilizing information from past queries. It uses these historical queries as surrogate gradients to estimate the expected improvement for each user. Specifically, we define the following rule for gradients update:

\[
\nabla \tilde{f}_i^t = \begin{cases}
    \nabla f_i( x_t ) & \text{if $i$ is selected at time $t$} \\
    \nabla \tilde{f}_i^{t-1} & \text{otherwise},
\end{cases}
\]
where $\nabla \tilde{f}_i^t$ is the surrogate gradient that stores the most recent queried gradient for each user. We further define the surrogate total gradient to be $\nabla \tilde{f}^t = \frac{1}{n} \sum_{i=1}^n \nabla \tilde{f}_i^t$.

%$$ \nabla \tilde{f}_i^t := \nabla f_i(x_{ \tau_i^t }), \quad \nabla \tilde{f}(x_t) := \frac{1}{n}\sum_{i=1}^n \nabla \tilde{f}_i(x_t),$$

Then, mimicking OGQ, SGQ will compute an approximated version of $\text{EI}_i$, denoted by $\tilde{\text{EI}}_i$:
\begin{equation}
\label{eq: EI_tilde}
    \tilde{\text{EI}}_i(x_t) := \alpha_t \langle \nabla \tilde{f}^t, \nabla \tilde{f}^t_i \rangle - \frac{\alpha_t^2 L}{2} \norm{\nabla \tilde{f}^t_i}^2. 
\end{equation}

The following lemma provides an upper bound for the error between the true expected improvement and the above approximated version.
\begin{lemma}
\label{lemma: r_i}
Define $\tau_i^t$ to be the time of the most recent query of user $i$, i.e. $\tau_i^t$ is picked such that $\nabla \tilde{f}_i^t = \nabla f_i( x_{ \tau_i^t })$. Let $\epsilon_t^{(i)} = L_i \norm{x_{\tau_i^t} - x_t}$, and $\overline{\epsilon}_t = \frac{1}{n} \sum_{j=1}^n \epsilon_t^{(j)}$, then  
$$
\begin{aligned}
\left|\text{EI}_i(x_t) - \tilde{\text{EI}}_i(x_t) \right| & \le r_i^t,
\end{aligned}
$$
where 
\begin{equation}
\begin{aligned}
\label{eq: r_i}
r_i^t & := \left(\alpha_t \norm{\nabla \tilde{f}^t} + \alpha_t^2 L \norm{\nabla \tilde{f}^t_i} \right ) \epsilon_t^{(i)} \\
    & + \alpha_t \norm{\nabla \tilde{f}^t_i} \overline{\epsilon}_t + \alpha_t \epsilon_t^{(i)} \overline{\epsilon}_t + \frac{\alpha_t^2 L}{2} (\epsilon_t^{(i)})^2.
    \end{aligned}
\end{equation}
\end{lemma}

Lemma \ref{lemma: r_i} indicates that the difference between $\text{EI}_i(x_t)$ and $ \tilde{\text{EI}}_i(x_t)$ has upper bound controlled by $r_i^t$. Therefore, when optimizing the true $\text{EI}$, we need to take this uncertainty $r_i^t$ into account.

% We then define the worst-case bound $r_i^t$ as the right-hand side of the above inequality, i.e. $\left|\text{EI}_i(x_t) - \tilde{\text{EI}}_i(x_t) \right| \le r_i^t$.

To maximize the true $\text{EI}$ under uncertainty about the true gradients, we adopt a bandit-type approach inspired by Upper Confidence Bound (UCB) methods \cite{slivkinsIntroductionMultiArmedBandits2019} to sequentially select gradient queries based on deterministic upper bounds $r_i^t$ and
$\tilde{\text{EI}}_i$. As such, in selecting the user $i_t$ to query, instead of choosing a user that maximizes the estimated $\tilde{\text{EI}}_i$, we maximize the upper bound of $\text{EI}_i$, i.e.
$$
i_t := \text{argmax}_i \left \{ \tilde{\text{EI}}_i(x_t) + r_i^t \right\},
$$ 
where $r_i^t $, as defined in Lemma \ref{lemma: r_i}, represents a worst-case upper bound of $\left|\text{EI}_i(x_t) - \tilde{\text{EI}}_i(x_t) \right|$ at time $t$. We will see that this UCB-style querying heuristic ensures that the regret of our approach relative to the baseline OGQ algorithm is bounded by  $2r_i^t$. Specifically, we have the following proposition:

\begin{proposition}
    \label{prop: UCB property}

    Let $i_t^*, i_t$ denotes the index picked by maximizing $\text{EI}_i(x_t)$ and  $\tilde{\text{EI}}_i(x_t) + r_i^t$, respectively, i.e. $$i_t^* := \text{argmax}_i \text{ EI}_i(x_t), \quad i_t := \text{argmax}_i \left \{ \tilde{\text{EI}}_i(x_t) + r_i^t \right\}.$$
    Then at time step $t$ we have,
    $$
    \text{EI}_{i_t^*}(x_t) - \text{EI}_{i_t}(x_t)  \le 2 r_t^i.
    $$

\end{proposition}

Note that if we further apply Lemma \ref{lemma: EI difference}, we will obtain
    $$
    \begin{aligned}
    \text{EI}_{i_t}(x_t) - \frac{1}{n} \sum_n \text{EI}_i(x_t) &\ge \text{EI}_{i_t^*}(x_t) - \frac{1}{n} \sum_n \text{EI}_i(x_t) -2r_i^t \\
    &\ge \sqrt{\frac{\mathrm{Var}[ \{ \text{EI}_i(x_t) \}_{i=1}^n ]}{c}} - 2r_i^t.
    \end{aligned}
    $$
    
Therefore, Proposition \ref{prop: UCB property} indicates that the per-step improvement of the SGQ algorithm is lower bounded by $\sqrt{\frac{\mathrm{Var}[ \{ \text{EI}_i(x_t) \}_{i=1}^n ]}{c}} - 2r_i^t$. As long as the $r_i^t$ is controlled to be sufficiently small, the SGQ algorithm guarantees improvement compared to classical SGD. We are now ready to introduce the SGQ algorithm.

\begin{algorithm2e}
	\caption{SGQ Algorithm}
	\label{alg: proposed}
	\SetAlgoNoLine
	\DontPrintSemicolon
	\LinesNumbered
	\KwIn{Stepsize $\alpha_t$, initial value $x_0$, Lipschitz constant $L_i$, $i \in [[1,n]]$, random explore probability $p$ }
	Initialize to query $ \nabla \tilde{f}_i^t = \nabla f_i(x_0)$, and set $\tau_i^t = 0$ for each $i$
 
	\For{$t=0, 1,2,\ldots,T-1$}{ 
        Set $\epsilon^{(i)}_t = L_i \norm{x_{\tau_i^t}-x_t}$, and $\overline{\epsilon_t} = \sum_{j=1}^n \epsilon_t^{(j)}/n$ \;
        \If{$\xi_t \sim \text{Uniform}(0,1) < p$}{
	        Select $i_t$ uniformly from $\{1, 2, \ldots, n\}$ %\tcp*{Random pick with probability $p$}
	    }
	    \Else{
	    Compute $\tilde{\text{EI}}_i(x_t)$ defined in   Equation \eqref{eq: EI_tilde}\;
            Compute $r_i^t$ defined in Equation \eqref{eq: r_i}
            \;
	        Find maximizer $i_t := \text{argmax}_i \{\tilde{\text{EI}}_i(x_t) + r_i^t\}$  \;
	    }
        Query user $i_t$ for true gradient $\nabla f_{i_t}(x_t)$ \;
        Update $\nabla \tilde{f}^t_{i_t} = \nabla f_{i_t} (x_t)$, and $\tau_{i_t}^t = t$\;
        Iterate $x_{t+1} = x_t - \alpha_t \nabla \tilde{f}^t_{i_t}$ \;}
        
        \KwOut{$x_T$}
\end{algorithm2e}

\subsection{Analysis for SGQ}

To analyze the SGQ algorithm, we impose an additional assumption that upper bounds the gradient difference between $f_i$'s:
\begin{assumption}
\label{assumption: Delta bound}
There exist a constant $\Delta >0 $ such that for any $i \in [[1,n]]$, we have
$$ \norm{\nabla f_i(x) - {\textstyle (1/n) \sum_{j=1}^n} \nabla f_j(x)} \le \Delta, \quad \forall x \in \R^d.$$
\end{assumption}

Assumption \ref{assumption: Delta bound} is a common assumption in the biased SGD literature \cite{demidovichGuideZooBiased2023}. Now we are ready to establish the performance guarantee for the SGQ algorithm.
\begin{theorem}[SGQ Algorithm]
    \label{thm: SGQ}
    Let Assumption \ref{assumption: f_i convex}-\ref{assumption: Delta bound} hold and define the optimality gap by $G_t := f(x_t) - \inf f$. If a constant stepsize is applied $\alpha_t \equiv \alpha$ and satisfies 
    $$
    \begin{aligned}\alpha \le \min\Bigg \{ & \frac{1 - \sqrt{1 - p/2n} }{L_{\max}},  \quad \frac{\mu}{4L L_{\max}}, \\
    &\quad \frac{p}{96n(L +L_{\max})}\cdot \frac{1}{1-p} \Bigg \}, 
    \end{aligned}$$
    then it holds that
    \begin{equation} 
    \label{eq: Proposed_bound}
    \begin{aligned}
        \E[ &G_t ] \le \left[1 -   \left( 1 + \frac{(1-p)(2\sqrt{C_1} + \sqrt{C_2})}{\sqrt{2c}} \right)\alpha \mu \right ]^t G_0  \\
         + &\frac{24\alpha (L + L_{\max})n}{p\mu } \frac{(1 - p) \sqrt{2c} }{\sqrt{2c} +(1-p)(2\sqrt{C_1} + \sqrt{C_2})}\cdot \Delta^2 \\
        +& \frac{\alpha L}{\mu} \frac{\sqrt{2c} - 2(1-p)\sqrt{C_2}}{\sqrt{2c} +(1-p)(2\sqrt{C_1} + \sqrt{C_2})} \sigma_f^* + \mathcal{O}((1-p)\alpha^2),
    \end{aligned}
    \end{equation}
    where $\sigma_f^* = \mathrm{Var}\left[ \{\nabla f_i(x^*)\}_{i=1}^n \right]$ is the gradient noise, and the constant hidden in $\mathcal{O}(\cdot)$ is a rational polynomial that depends on $\sqrt{C_1}, \sqrt{C_2}, \mu, \sqrt{c}, \Delta, L, L_{\max}$. 
\end{theorem}

This result generalizes the analysis in \eqref{SGD_bound} for SGD: If we set the random explore probability $p=1$ (i.e. querying by uniformly sampling), the performance guarantee of SGQ naturally recovers \eqref{SGD_bound}. The key step of Theorem \ref{thm: SGQ} is to control the cumulative regret $\sum_{k=1}^t 2 r_{i_k}^t$ with respect to the OGQ algorithm. This estimation is achieved by first bounding each $\norm{x_{\tau_i^t} - x_t}$ given the updating rule of SGQ, and then apply Lemma \ref{lemma: r_i} to obtain the bound for $r_{i_k}^t$. The details of the proof can be found in the appendix.

Theorem \ref{thm: SGQ} shows that SGQ achieves a faster exponential decaying rate in the transient-state, similar to OGQ, though it may introduce an additional variance term. However, in certain scenarios, this increase of variance is acceptable. For example, when the optimization task does not require a high-precision solution but rather a low-precision one, SGQ can achieve the same level of precision with a strictly smaller number of queries compared to SGD. This is particularly advantageous in settings where querying is costly, as SGQ reduces the number of queries needed to reach an acceptable solution while a further precise solution is not necessary. In such cases, the benefit of strategic querying becomes evident: it accelerates transient-state convergence, enabling the algorithm to reach a satisfactory solution more quickly while minimizing the overall querying cost. By focusing on more informative users in each query, we have proved that SGQ effectively utilizes the heterogeneity between users to optimize the objective function more efficiently, making it a powerful alternative to traditional methods like SGD, especially in applications with expensive data acquisition cost.

Furthermore, in our numerical simulations, we observe that SGQ does not necessarily lead to increased steady-state noise. Instead, it maintains or even reduces the noise level, which will be discussed in the next section.

\begin{remark}
We note that there are indeed algorithms like SAGA \cite{defazioSAGAFastIncremental2014} that achieve linear convergence with only one query per iteration. However, there are two key distinctions that set our work apart from algorithms like SAGA: 1) Although they achieve linear convergence, they do not offer any acceleration in transient-state compared to SGD, as will be demonstrated in the numerical simulations in the next section; 2) As discussed earlier, the development of OGQ/SGQ is not intended to propose a new variant of SGD that competes with a specific state-of-the-art algorithm. Rather, our goal is to establish a broader framework for understanding how querying strategies can accelerate the transient performance of stochastic gradient-based methods.
Moreover, we also observe numerically that it is possible to equip algorithms like SAGA with a querying strategy that improves their transient-state performance while still preserving its linear convergence rate.

% While our primary motivation for excluding algorithms like SAGA is to isolate the impact of the querying strategy, our problem setting also has practical significance. For instance, in certain online learning scenarios where the objective function is time-varying, algorithms like SAGA become inapplicable due to their update rules that descend in the direction of historical queries.

\end{remark}

\section{Numerical Experiments} 
\label{section: numerical}
In this section, we present a numerical experiment to evaluate the performance of the Oracle Gradient Querying (OGQ) and Strategic Gradient Querying (SGQ) algorithms in comparison with SGD, SAGA, as well as SVRG. Our goal is to demonstrate the performance of our proposed algorithms and illustrate the benefits of strategic querying in finite-sum optimization.

We will compare algorithms in a 1-D toy example, where we have four component functions $f_1(x) = (x-2)^2$, $f_2(x) = (x-1)^2$, $f_3(x) = (x+1)^2, f_4(x) = (x+2)^2$, and the overall objective function is given by: $f(x) = \frac{1}{4}\sum_{i=1}^4 f_i(x) = x^2 + 5/2$. In this setting, we compare the convergence behavior of SGD, OGQ, and SGQ, tracking their objective error over iterations. In specific, we set the initial point at $x_0 = 5$, the step size to be a constant $\alpha = 1.5 \times 10^{-2}$. 

Figure~\ref{fig: whole comparison} presents a performance comparison between OGQ, SGQ, and several baseline algorithms, including SGD, SAGA, and SVRG. For this experiment, the SGQ hyperparameter $p$ is set to $0.3$, and SVRG computes the full gradient once every 10 iterations. The plot shows the expected objective error, defined as $\mathbb{E}[f(x_t)] - \inf f$, where the expectation is estimated by averaging over 200 independent trials. The x-axis represents the total number of queries, which shows a direct comparison of querying efficiency across algorithms.\footnote{It is worth noting that OGQ queries only once per iteration by selecting the most “informative” user. However, to maximize the expected improvement (EI), OGQ assumes access to the full gradient information at each step. As such, OGQ should be regarded solely as a benchmark algorithm that makes idealized decisions based on complete information. It is not intended for practical deployment due to its unrealistic information requirements.}
As expected, OGQ achieves the transient acceleration while also reducing variance (compared to SGD) in the steady state, thanks to its oracle-based querying strategy. Notably, although SAGA and SVRG are theoretically proven to achieve linear convergence, our SGQ algorithm significantly outperforms SGD, SAGA, and SVRG in terms of performance during the transient state. In particular, SGQ reaches the same level of precision using roughly half the number of queries required by SGD and SAGA.

Figure \ref{fig: p value} shows the performance of SGQ under different values of the hyperparameter $p$. A smaller $p$ indicates less random exploration, which, according to Theorem \ref{thm: SGQ}, results in better transient-state performance but may lead to higher steady-state variance. On the other extreme, when $p$ is close to 1, SGQ behaves similarly to SGD by applying random sampling most of the time. In Figure \ref{fig: p value}, we observe that as $p$ decreases, the transient state of SGQ is accelerated and increasingly approaches the transient-state performance of OGQ.

These results confirm that, a well-designed querying strategy can accelerate transient-state for SGD, supporting the theoretical findings in Theorems \ref{thm: OGQ} and \ref{thm: SGQ}. Additionally, we observe that SGQ exhibits lower steady-state variance than SGD in our numerical experiment. This suggests that the mechanism behind SGQ’s querying strategy may inherently contribute to variance reduction. As a promising future research direction, we aim to further investigate the underlying reasons for this variance reduction and its broader implications.

\begin{figure}
    \centering
    \includegraphics[width=0.8\linewidth]{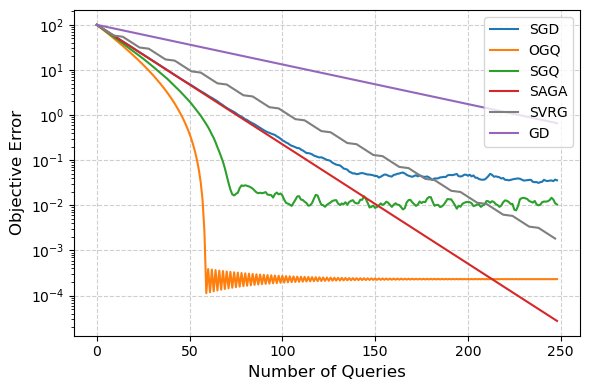}
    \caption{The averaged objective error comparison between SGD, OGQ, SGQ, SAGA, and SVRG.}
    \label{fig: whole comparison}
    \vspace{-0.4cm}
\end{figure}

\begin{figure}
    \centering
    \includegraphics[width=0.8\linewidth]{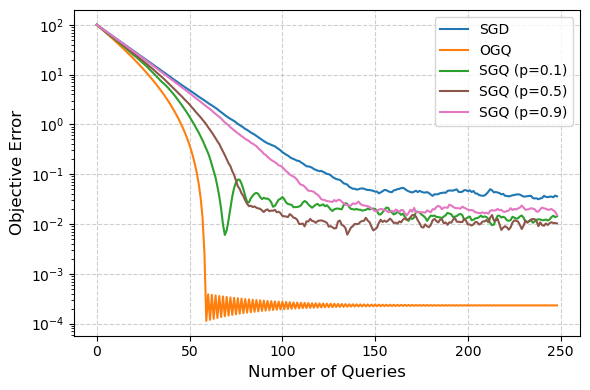}
    \caption{The comparison of SGQ with different exploration parameter $p$.}
    \label{fig: p value}
    \vspace{-0.4cm}
\end{figure}

\section{Conclusion} 
\label{section: conclusion}

This paper considers the finite-sum optimization problem with first-order queries and investigates the impact of strategic querying on querying efficiency. We first introduce Oracle Gradient Querying (OGQ) as an upper-bound benchmark that selects the most informative query at each step using full gradient knowledge. We then propose Strategic Gradient Querying (SGQ), a practical algorithm that mimics OGQ’s querying strategy while maintaining the same number of queries as standard SGD. Our theoretical analysis proves that OGQ and SGQ both improve transient-state convergence, and OGQ additionally reduces the steady-state variance. Our numerical experiments further show that SGQ achieves faster convergence while also potentially reducing the steady-state variance in practice.

Looking ahead, there are several directions that we would like to explore in the future. First, we can further explore the theory to support why SGQ can potentially reduce the steady-state variance in practice.  Second, we can extend our study by comparing strategic querying with other stochastic gradient-based algorithms, such as accelerated SVRG \cite{driggsAcceleratingVariancereducedStochastic2022} to further demonstrate the impact of strategic querying. We can also generalize our strategic querying analysis to the non-convex optimization setting.

\section{Supplementary Discussion}
\label{section: justification}

In this section, we provide discussions that further justify the assumptions that we made in previous sections. 

\subsection{Determine $C_1$, $C_2$ in Assumption \ref{assumption: Hete}}

This assumption states that the variance of EI scales with the gradient norms $\norm{\nabla f}$ and $\frac{1}{n}\sum_{i=1}^n\norm{\nabla f_i}^2 $, where $C_1$ and $C_2$ serve as scaling coefficients. A simple example of $\{f_i\}_{i=1}^n$ satisfying this requirement is the class of quadratic functions: $f_i(x) = a_i(x - b_i)^2$, where we require $n > 3$ and the coefficients $\{a_i\}_{i=1}^n$, $\{b_i\}_{i=1}^n$ are not identical, respectively.

We then present the following result to characterize possible values of $C_1$ and $C_2$. For simplicity, we focus on the case where $x$ is scalar-valued, leaving the study of higher-dimensional cases for future work. Specifically, the formulation of $C_1$ and $C_2$ is given by the following proposition:

\begin{proposition} 
\label{prop: C_1, C_2}
Let $ x \in \mathbb{R}$ be scalar valued and denote the empirical mean, variance, skewness, and kurtosis of $ \{ f_i'(x) \} $ with respect to $i$ as $ \bar{g}(x) $, $ V_g(x) $, $ S_g(x) $, and $ \kappa_g(x) $, respectively. Suppose the step size $ \alpha $ is chosen such that $ \alpha \leq \frac{1}{2L} $, then we can set 
$$
C_1 = \inf_{x\in \R} \frac{\delta(x) V_g(x)}{4 \bar{g}^2(x)}, \quad C_2 = \inf_{x\in \R}\frac{\delta(x) (\kappa_g(x) - 1)V_g(x)^2}{4\left(V_g(x) + \bar{g}^2(x) \right)^2}.
$$ 
where $\delta(x) = 1 - \frac{S_g(x)}{\sqrt{\kappa_g(x) -1}}$ and the positivity of $ \delta(x) $ is ensured by one of the Pearson's inequalities \cite{sharmaSkewnessKurtosisNewtons2015}. 

% Furthermore, if for each user $i$, $f_i'(x)$ is independently sampled from a fixed distribution and $ n $ is sufficiently large, then by the central limit theorem we know that, $ \delta(x) $ and $ \kappa_g(x) $ will converge to positive values for all $ x $. \htwai{Here, you are describing yet another sampling procedure (different from $i$) that involves the values of $f_i'(x)$ itself. be careful.} For instance, if the underlying distribution of sampling $f_i'(x)$ is normal, we obtain $ \delta = 1 $ and kurtosis $ \kappa = 3 $. 

% If we further assume that there exists a uniform lower bound for $\delta(x), \kappa_g(x)$, i.e. there exist a $\delta > 0$, $\kappa_g > 1$ such that
\end{proposition}

We note that the infimum operation in the definitions of $C_1$ and $C_2$ serves solely to eliminate their dependence on $x$. An alternative approach would be to define $C_1$ and $C_2$ locally with respect to $x$. In this case, Assumption \ref{assumption: Hete} would still hold but would yield different acceleration rates across different regions of $x$. The proof of Proposition \ref{prop: C_1, C_2} can be found in the appendix.

\subsection{Tighter Bound of $\tilde{c}(x)$}
Next, we provide results on estimating tighter bounds for $\tilde{c}(x)$. The key intuition that why a tighter bound is possible is that, in practical scenarios, the functions $f_i$ are usually not arbitrarily assigned. As a result, the extreme case where $\tilde{c}(x) = n-1$ is highly unlikely. Therefore, if we assume that the functions $f_i(x)$'s are independently drawn from certain underlying distribution (e.g. $f_i(x) = g(x,\xi_i)$ for some $g$ and $\xi_i$'s are i.i.d samples from some distribution). Then as a consequence, $\text{EI}_i(x)$ can also be treated as samples under another certain distribution. This allows us to further analyze the distribution of $\tilde{c}(x)$. The following proposition shows that if we denote the underlying distribution of $\text{EI}_i(x)$ by $\D_x$, then $\tilde{c}(x)$ may enjoy a tighter bound, depending on the property of $\D_x$. 

\begin{proposition}
\label{prop: tilde_c}
    Let $\text{EI}_1(x), \text{EI}_2(x), ..., \text{EI}_n(x)$ be variables sampled i.i.d from non-degenerate distribution $\D_x$ with finite mean and variance. Then there exists a scale coefficient $\tilde{c}(x)$ such that 
    $$\frac{ \frac{1}{n} \sum_{i=1}^n \text{EI}_i(x)  - \min_i \text{EI}_i(x) }{\max_i \text{EI}_i(x) - \frac{1}{n} \sum_{i=1}^n \text{EI}_i(x) } \le \tilde{c}(x),$$
     where the scale coefficient $\tilde{c}(x)$ has the following property:

    \begin{itemize}
        \item $\tilde{c}(x) \le n-1$.
        \item If $\D_x$ is sub-gaussian with variance proxy $\sigma_{\D_x}^2$, then $\tilde{c}(x) \le \mathcal{O}(\sigma_{\D_x}\sqrt{\log n})$ with probability $1 - 1/n$.
        \item If $\D_x$ is bounded by constant $B_{\D}$ (i.e. $|\text{EI}_i(x)| \le B_{\D_x}$), then $\tilde{c}(x) \le \mathcal{O}(B_{\D_x})$ with probability $1 - 1/n$.
    \end{itemize}
   
    The constant hidden in $\mathcal{O}(\cdot)$ solely depend on the density function of $\D_x$. 
\end{proposition}
Proposition \ref{prop: tilde_c} indicates that if $\text{EI}_i(x)$ can be characterized as samples from certain distributions, then $\tilde{c}(x)$ can have a tighter upper bound that doesn't grow or grows slowly as $n$ increases.

\section{Acknowledgements}
This work was supported by NSF grant \#2419982.

\bibliographystyle{IEEEtran}
\bibliography{IEEEabrv,references}

\newpage
\onecolumn

\appendix

\subsection{Proof of Main Results}
% For notational simplicity, in this appendix, we use $\E_i [\cdot]$ (defined as $\E_i[y_i] = \frac{1}{n} \sum_{i=1}^n y_i$) to represent taking empirical expectation with respect to index $i$. Meanwhile, the notation $\E[X]$, $\E[X|Y]$ represents taking the expectation of $X$ and taking expectation of X conditioned on $Y$.

\noindent \textbf{Proof of Theorem \ref{thm: OGQ}}
\begin{proof} 
We first prove the convergence result for SGD under given assumptions. The proof follows closely from the classical analysis for SGD, with subtle notational differences. 

By Assumption \ref{assumption: f_i smooth} and by the definition of expected improvement, we obtain that
\begin{equation}
\label{eq: SGD analysis}
\begin{aligned}
    f(x_t) &\ge \E [f(x_{t+1})| x_t] +  \E_i \left[\text{EI}_i(x_t) \right] \\
    &= \E\left[f(x_{t+1})|x_t\right ] + \alpha \norm{\nabla f (x_t)}^2 - \frac{\alpha^2 L}{2}\E_i \left[\norm{\nabla f_i(x_t)}^2\right],
\end{aligned}
\end{equation}
where $\E_i [\cdot]$ denotes that we are taking expectation with respect to index $i$.

Then, by Assumption \ref{assumption: f PL-condition} we have 
$$
\begin{aligned}
    f(x_t) &\ge \E\left[f(x_{t+1})| x_t\right ] + \alpha \norm{\nabla f (x_t)}^2 - \frac{\alpha^2 L}{2}\E_i \left[\norm{\nabla f_i(x_t)}^2\right] \\
    &\ge \E\left[f(x_{t+1})|x_t\right ] + \alpha \norm{\nabla f (x_t)}^2 - \left (2\alpha^2LL_{\max}(f(x_t) - \inf f) + \alpha^2L\sigma_f^* \right) \\
    &\ge \E\left[f(x_{t+1})|x_t\right ] + \alpha (2\mu - 2\alpha LL_{\max})(f(x_t) - \inf f) - \alpha^2L\sigma_f^* \\
    &\ge \E\left[f(x_{t+1})| x_t\right ] + \alpha \mu (f(x_t) - \inf f) - \alpha^2L\sigma_f^* 
\end{aligned}
$$
where the second inequality use the well known variance transfer lemma:
$$ \E_i \left[\norm{\nabla f_i(x)}^2 \right] \le 4L_{\max} (f(x) - \inf f) + 2\sigma_f^*,$$
and the third inequality uses Assumption \ref{assumption: f PL-condition}.
Then, for the last inequality, we used our stepsize assumption $\alpha \le \frac{\mu}{2L L_{\max}}$, so that $2\mu - 2\alpha LL_{\max} \ge \mu$.

Then we subtract $\inf f$ on both side and re-organize the equation. Finally by take expectation on $x_t$ we obtain
$$
\E[f(x_{t+1}) - \inf f] \le (1 - \alpha \mu) \E[f(x_t) - \inf f] + \alpha^2 L \sigma_f^*.
$$

By recursively applying the above inequality and telescoping them together, we have
$$
\E[f(x_{t}) - \inf f] \le (1 - \alpha \mu)^t \E[f(x_0) - \inf f] +  \frac{\alpha L}{\mu} \sigma_f^*,
$$
where $\sum_{i=0}^{t-1} (1- \alpha \mu) \le \frac{1}{\alpha \mu}$ is used to simplify the geometric sum.

Next, we prove the convergence result for the OGQ algorithm. The difference between SGD and Ideal algorithm is that, the Ideal algorithm always maximize $\text{EI}_i(x_t)$ at each step, instead of receiving the averaged improvement. Thus similar to \eqref{eq: SGD analysis}, for OGQ algorithm, we have
$$
\begin{aligned}
f(x_t) - f(x_{t+1})&\ge \max_i \text{EI}_i(x_t) \\
&= \E_i[\text{EI}_i(x_t)] + \left(\max_i \text{EI}_i(x_t) - \E_i[\text{EI}_i(x_t)] \right)  \\
\text{[Result from \eqref{eq: SGD analysis}]} \quad & \ge \alpha \norm{\nabla f (x_t)}^2 - \frac{\alpha^2 L}{2}\E_i \left[\norm{\nabla f_i(x_t)}^2\right] + \left(\max_i \text{EI}_i(x_t) - \E_i[\text{EI}_i(x_t)] \right) \\
\text{[By Corollary \ref{lemma: EI difference}]} \quad & \ge \alpha \norm{\nabla f (x_t)}^2 - \frac{\alpha^2 L}{2}\E_i \left[\norm{\nabla f_i(x_t)}^2\right] + \sqrt{\frac{\mathrm{Var}[\{\text{EI}_i(x_t)\}_{i=1}^n]}{c}} \\
\text{[By Assumption \ref{assumption: Hete}]} \quad & \ge \alpha \norm{\nabla f (x_t)}^2 - \frac{\alpha^2 L}{2}\E_i \left[\norm{\nabla f_i(x_t)}^2\right] + \sqrt{\frac{C_1 \alpha^2 \norm{\nabla f(x_t)}^4 + C_2 \alpha^4 L^2 \cdot \E_i \left [\norm{\nabla f_i(x_t)}^2 \right ]^2}{c}} \\
& \ge \alpha \norm{\nabla f (x_t)}^2 - \frac{\alpha^2 L}{2}\E_i \left[\norm{\nabla f_i(x_t)}^2\right] + \sqrt{\frac{C_1}{2c}}\alpha \norm{\nabla f(x_t)}^2 + \sqrt{\frac{C_2}{2c}} \alpha^2 L \E_i \left [\norm{\nabla f_i(x_t)}^2 \right ] \\
&= \left(1 + \sqrt{\frac{C_1}{2c}} \right ) \alpha \norm{\nabla f(x_t)}^2 - \frac{\alpha^2 L}{2} \left( 
         1 - \sqrt{\frac{2C_2 }{c}} \right ) \E_i\left[\norm{\nabla f_i(x_t)}^2\right] \\
&\ge \left(1 + \sqrt{\frac{C_1}{2c}} \right ) \alpha \cdot 2\mu(f(x_t) - \inf f) - \frac{\alpha^2 L}{2} \left( 
         1 - \sqrt{\frac{2C_2 }{c}} \right ) (4L_{\max} (f(x_t) - \inf f) + 2\sigma_f^*),
\end{aligned}
$$
where the last inequality uses a well-known result called variance transfer lemma, which is stated and proved in Lemma \ref{lemma: variance_transfer}.

Then by subtracting $\inf f$ on both side and re-organizing the equation, we have
$$
\begin{aligned}
f(x_{t+1}) - \inf f &\leq \left[ 1 - \left( 1 + \sqrt{\frac{C_1}{2 c}} \right) \alpha 2 \mu + 2 \alpha^2 L L_{\max} \left( 1 - \sqrt{\frac{2 C_2}{c}} \right) \right] (f(x_t) - \inf f) + \alpha^2 L \left( 1 - \sqrt{\frac{2 C_2}{c}} \right) \sigma_f^* \\
&= \left[ 1 - \left( 2 + \sqrt{\frac{2C_1}{c}} - \alpha \frac{2L L_{\max}}{\mu} \left( 1 -  \sqrt{\frac{2 C_2}{c}} \right)  \right)\alpha \mu \right] (f(x) - \inf f) + \alpha^2 L \left( 1 - \sqrt{\frac{2 C_2}{c}} \right) \sigma_f^* \\
\end{aligned}
$$
By using our stepsize assumption $\alpha \le \frac{\mu}{2 L L_{\max}}$, we obtain that 
$$
f(x_{t+1}) - \inf f \le \left[ 1 - \left( 1 + \frac{\sqrt{2}(\sqrt{C_1} + \sqrt{C_2})}{\sqrt{c}} \right) \alpha \mu\right] (f(x) - \inf f) + \alpha^2 L \left( 1 - \sqrt{\frac{2 C_2}{c}} \right) \sigma_f^* 
$$

Then by recursively applying all the inequality from 1 to $t$, we have
$$
\begin{aligned}
f(x_t) - \inf f &\leq \left[ 1 - \left( 1 + \frac{\sqrt{2}(\sqrt{C_1} + \sqrt{C_2})}{\sqrt{c}} \right) \alpha \mu \right]^t (f(x_0) - \inf f) + \alpha^2 L \cdot \frac{\left( 1 - \sqrt{\frac{2 C_2}{c}} \right) }{1 + \frac{\sqrt{2}(\sqrt{C_1} + \sqrt{C_2})}{\sqrt{c}}}\cdot \frac{\sigma_f^*}{\alpha \mu}\\
&= \left[ 1 - \left( 1 + \frac{\sqrt{2}(\sqrt{C_1} + \sqrt{C_2})}{\sqrt{c}} \right) \alpha \mu \right]^t (f(x_0) - \inf f) + \alpha \frac{L}{\mu} \cdot \frac{\sqrt{c/2 } - \sqrt{C_2} }{\sqrt{c/2 } + (\sqrt{C_1} + \sqrt{C_2})}\cdot \sigma_f^*.\\
\end{aligned}
$$

This matches with the result in Equation \eqref{eq: Ideal_bound} in Theorem \ref{thm: OGQ}.
\end{proof}

\noindent \textbf{Proof of Lemma \ref{lemma: r_i}}
\begin{proof} 
From the definition of $\text{EI}_i$ and $\tilde{\text{EI}_i}$, we have
    $$ 
\begin{aligned}|\tilde{\text{EI}_i}(x_t) &- \text{EI}_i(x_t) | = \left | (\alpha_t \langle \nabla \tilde{f}^t, \nabla \tilde{f}^t_i\rangle - \alpha_t \langle \nabla f(x_t), \nabla f_i(x_t) \rangle)- \frac{\alpha_t^2 L}{2} (\norm{\nabla \tilde{f}^t_i}^2 - \norm{\nabla f_i(x_t)}^2 )  \right| \\
&\le \left | \alpha_t \langle \nabla \tilde{f}^t, \nabla \tilde{f}_i^t\rangle - \alpha_t \langle \nabla f(x_t), \nabla f_i(x_t) \rangle \right| +  \frac{\alpha_t^2 L}{2} \left | \norm{\nabla \tilde{f}_i^t}^2 - \norm{\nabla f_i(x_t)}^2 \right| \\
&= \left | \alpha_t \langle \nabla \tilde{f}^t - \nabla f(x_t), \nabla \tilde{f}_i^t \rangle - \alpha_t \langle \nabla f(x_t), \nabla f_i(x_t) -\nabla \tilde{f}_i^t\rangle \right| +  \frac{\alpha_t^2 L}{2} \left | \norm{\nabla \tilde{f}_i^t}^2 - \norm{\nabla f_i(x_t)}^2 \right| \\
&\le \alpha_t \norm{\nabla \tilde{f}_i^t} \cdot \norm{\nabla \tilde{f}^t - \nabla f(x_t)}  + \alpha_t \left ( \norm{\nabla \tilde{f}^t} +\norm{\nabla \tilde{f}^t - \nabla f(x_t)} \right ) \cdot \norm{\nabla f_i(x_t) - \nabla \tilde{f}_i^t }  \\
&\quad + \alpha_t\norm{\nabla \tilde{f}_i^t - \nabla f_i(x_t)} + \frac{\alpha_t^2 L}{2} \norm{\nabla \tilde{f}_i^t + \nabla f_i(x_t)} \norm{\nabla \tilde{f}_i^t - \nabla f_i(x_t)} \\
&\le \alpha_t \norm{\nabla \tilde{f}_i^t} \cdot \norm{\nabla \tilde{f}^t - \nabla f(x_t)} + \alpha_t \norm{\nabla \tilde{f}^t} \cdot \norm{\nabla \tilde{f}_i^t - \nabla f_i(x_t)}  \\
&\quad + \alpha_t\norm{\nabla \tilde{f}^t - \nabla f(x_t)}\norm{\nabla \tilde{f}_i^t - \nabla f_i(x_t)} + \frac{\alpha_t^2 L}{2} \left(2\norm{\nabla \tilde{f}_i^t} + \norm{\nabla \tilde{f}_i^t - \nabla f_i(x_t)} \right)\norm{\nabla \tilde{f}_i^t - \nabla f_i(x_t)} \\
&\le \alpha_t \norm{\nabla \tilde{f}_i^t} \cdot \overline{\epsilon}_t + (\alpha_t \norm{\nabla \tilde{f}^t} + \alpha_t^2 L \norm{\nabla \tilde{f}^t_i}) \epsilon_t^{(i)} + \alpha_t \epsilon_t^{(i)} \overline{\epsilon}_t + \frac{\alpha_t^2 L}{2} (\epsilon_t^{(i)})^2 \quad =: r_i^t
\end{aligned}
$$
where the first and third inequality uses the triangle inequality, while the second inequality is from the Cauchy-Schwartz inequality. The last inequality is a directly applying the $L_i$-smoothness for each $f_i$.
\end{proof}

\noindent \textbf{Proof of Theorem \ref{thm: SGQ}}

To prove the convergence result in Theorem \ref{thm: SGQ}, we first q the following lemma and proposition that relates $x_{t-k}$ to $x_t$. 

    \begin{lemma}
    \label{lemma: x_(t-k) - x_t}
    For any gradient-based method with update rule as $x_{t+1} = x_{t} - \alpha \nabla f_{i_t} (x_t)$ that solves problem \eqref{eq: finite sum}. Let Assumption \ref{assumption: f_i smooth} holds for $f_i$'s. If the stepsize is chosen $\alpha < \frac{1}{L_{\max}}$, we have for $t\ge1$,

    $$ \frac{1}{\alpha}\norm{x_{t-1} - x} = \norm{\nabla f_{i_{t-1}} (x_{t-1})} \le \frac{1}{1 - \alpha L_{i_{t-1}}} \norm{\nabla f_{i_{t-1}} (x_t)}, $$
    and for arbitrary $i \in [[1,n]]$, we have
    $$ \norm{\nabla f_{i} (x_{t-1})} \le \norm{\nabla f_{i} (x_t)} + \frac{\alpha L_i}{1 - \alpha L_{i_{t-1}}} \norm{\nabla f_{i_{t-1}} (x_t)}.
    $$
\end{lemma}

\begin{proof}
    By triangle inequality, we have
    $$
    \begin{aligned} 
    \norm{\nabla f_{i_{t-1}} (x_{t-1})} &\le \norm{\nabla f_{i_{t-1}} (x_{t})} + \norm{\nabla f_{i_{t-1}} (x_{t-1}) - \nabla f_{i_{t-1}} (x_{t})} \\
    &\le \norm{\nabla f_{i_{t-1}} (x_{t})} + L_{i_{t-1}} \norm{x_{t-1} - x_t} \\
    &= \norm{\nabla f_{i_{t-1}} (x_{t})} + L_{i_{t-1}} \norm{\alpha \nabla f_{i_{t-1}}(x_{t-1})},
    \end{aligned}
    $$
    where the last equality uses the updating rule $x_{t+1} = x_{t} - \alpha \nabla f_{i_t} (x_t)$. 
    
    Thus we have $(1 - \alpha L_{i_{t-1}}) \norm{\nabla f_{i_{t-1}} (x_{t-1})} \le \norm{\nabla f_{i_{t-1}} (x_{t})}$. Divide both sides by $(1 - \alpha L_{i_{t-1}})$, we prove the first argument.

Additionally, as a simple extension, we have that for arbitrary index $i$:
$$ 
\begin{aligned}
\norm{\nabla f_{i} (x_{t-1})} &\le \norm{\nabla f_{i} (x_t)} + \norm{\nabla f_i (x_{t-1}) - \nabla f_{i} (x_{t})} \\
&\le \norm{\nabla f_{i} (x_t)} + L_i \norm{x_{t-1} - x_t}\\
&= \norm{\nabla f_{i} (x_t)} + \alpha L_i \norm{\nabla f_{i_{t-1}} (x_{t-1})} \\
&\le \norm{\nabla f_{i} (x_t)} + \frac{\alpha L_i}{1 - \alpha L_{i_{t-1}}} \norm{\nabla f_{i_{t-1}} (x_t)}, \\
\end{aligned}
$$
which proves the second argument.
\end{proof}

\begin{corollary}
\label{coro: x_(t-k) - x_t}
Denote $B_k = \norm{x_{t-k} - x_t}$, then under the assumption of Lemma \ref{lemma: x_(t-k) - x_t}, then we have the following estimation on $B_k$:

$$
\begin{aligned}
B_k = \norm{x_{t-k} - x_t} &\le \frac{\alpha}{1-\alpha L_{\max }}\left(\frac{1}{1-\alpha L_{\max }}\right)^{k-1} \cdot \sum_{l=1}^k \norm{\nabla f_{i_{t-l}}(x_t)}
\end{aligned}
$$
\end{corollary}

\begin{proof}
The proof is done by recursively using the result in Lemma \ref{lemma: x_(t-k) - x_t}. First, we have
$$
\begin{aligned}
    B_k &= \norm{x_{t-k} - x_t} = \alpha \norm{\nabla f_{i_{t-1}} (x_{t-1}) + ... + \nabla f_{i_{t-k}} (x_{t-k})} \\
    & \le B_{k-1} + \alpha \norm{\nabla f_{i_{t-k}} (x_{t-k})} \\
    & \le B_{k-1} + \alpha \norm{\nabla f_{i_{t-k}} (x_t)} + \alpha \norm{\nabla f_{i_{t-k}} (x_t) - \nabla f_{i_{t-k}} (x_{t-k})} \\
    & \le B_{k-1} + \alpha \norm{\nabla f_{i_{t-k}} (x_t)} + \alpha L_{i_{t-k}}B_{k} 
\end{aligned}
$$
Thus we have the following iteration rule:
$$B_k \le \frac{1}{1 - \alpha L_{i_{t-k}}} B_{k-1} + \frac{\alpha}{1 - \alpha L_{i_{t-k}}} \norm{\nabla f_{i_{t-k}} (x_t)}. $$

Recursively using the above relationship, we will have
$$
\begin{aligned}
B_k = \norm{x_{t-k} - x_t} &\le \frac{\alpha}{1 - \alpha L_{i_{t-k}}} \norm{\nabla f_{i_{t-k}} (x_t)} + \frac{\alpha}{1 - \alpha L_{i_{t-k+1}}}\left( \frac{1}{1 - \alpha L_{i_{t-k+1}}} \right) \norm{\nabla f_{i_{t-k+1}} (x_t)} + ... \\
&+ \frac{\alpha}{1 - \alpha L_{i_{t-1}}}  \prod_l (\frac{1}{1 - \alpha L_{i_{t-l}}}) \norm{\nabla f_{i_{t-1}} (x_t)} \\
&\le \frac{\alpha}{1-\alpha L_{\max }}\left(\frac{1}{1-\alpha L_{\max }}\right)^{k-1} \cdot \sum_{l=1}^k \norm{\nabla f_{i_{t-l}}(x_t)}
\end{aligned}
$$
\end{proof}

Now we are ready to prove Theorem \ref{thm: SGQ}.
\begin{proof}

Since at each time step, the SGQ algorithm will do a random exploration with probability $p$, the analysis of the SGQ algorithm can be separated into two parts:

1. If at time $t$, the random exploration is triggered, then the per-step increase is exactly the same as SGD:
$$ 
\begin{aligned}
f(x_t) - \E[f(x_{t+1})|x_t] &\ge \alpha \norm{\nabla f(x_t)}^2 - \frac{\alpha^2 L}{2} \E_i\left[\norm{\nabla f_i(x_t)}^2\right] 
\end{aligned}
$$

2. Otherwise, the algorithm will choose to iterate with a UCB-type approach. Then, through the regret analysis compared with Ideal algorithm, the per-step increase can be written as

$$ 
\begin{aligned}
f(x_t) - f(x_{t+1}) &\ge \left(1 + \sqrt{\frac{C_1}{2c}} \right ) \alpha \norm{\nabla f(x_t)}^2 - \frac{\alpha^2 L}{2} \left( 1 - \sqrt{\frac{2C_2 }{c}} \right ) \E_i\left[\norm{\nabla f_i(x_t)}^2\right] - 2 r^t_{i_t},
\end{aligned}
$$
where $r^t_i$ denotes the worst-case bound of index $i$ at time $t$, with the definition given in Algorithm \ref{alg: proposed} as:
$$r^t_i = \alpha \norm{\nabla \tilde{f}_i} \overline{\epsilon}_t + \alpha \norm{\nabla \tilde{f}}\epsilon_t^{(i)} + \alpha^2 L \norm{\nabla \tilde{f}_i}\epsilon_t^{(i)}  + \alpha \epsilon_t^{(i)} \overline{\epsilon}_t + \frac{\alpha^2 L}{2} (\epsilon_t^{(i)})^2. $$

Thus by taking expectation on the random exploration variable $\xi_t$, we obtain that the expected per-step increase is of the following:
$$ 
\begin{aligned}
\E[f(x_t) - & f(x_{t+1})|x_t] =
p\cdot \E \Big[f(x_t) - f(x_{t+1})\Big | x_t, \xi_t \le p \Big ] + (1-p)\cdot \E \Big[f(x_t) - f(x_{t+1})\Big | x_t, \xi_t > p \Big ]\\ 
& \ge \left(1 + (1-p)\sqrt{\frac{C_1}{2c}} \right ) \alpha \norm{\nabla f(x_t)}^2 - \frac{\alpha^2 L}{2} \left( 1 - (1-p)\sqrt{\frac{2C_2 }{c}} \right ) \E_i\left[\norm{\nabla f_i(x_t)}^2\right] - 2(1-p)r^t_{i_t},
\end{aligned}
$$

It suffices to provide an bound for $\E[r_i^t]$. Firstly, we notice that
$$
\begin{aligned}
& \norm{\nabla \tilde{f}^t_i}=\norm{\nabla f_i(x_{\tau_i^t})}\leq \norm{\nabla f_i (x_t)}+\norm{\nabla f_i (x_{\tau_i^t})-\nabla f_i(x_t)} \\
& \leq \norm{\nabla f_i(x_t)} +L_i\norm{x_{\tau_i^t}-x_t}=\norm{\nabla f_i(x_t)}+\epsilon_t^{(i)}.
\end{aligned}
$$

Similarly we also have $\norm{\nabla \tilde{f}^t} = \norm{\nabla f(x_t)} + \overline{\epsilon}_t$. Thus,

\begin{equation}
\label{eq: r_i estimate}
\begin{aligned}
r_i^t &= \alpha\norm{\nabla \tilde{f}^t_i} \cdot \overline{\epsilon}_t+\alpha\|\nabla \tilde{f}^t\| \epsilon_t^{(i)}+\alpha^2 L\norm{\nabla \tilde{f}^t_i } \epsilon_t^{(i)}+\alpha \overline{\epsilon}_t \epsilon_t^{(i)}+\frac{\alpha^2 L}{2}\left(\epsilon_t^{(i)}\right)^2 \\
& \leqslant \alpha\left(\norm{\nabla f_i}+\epsilon_t^{(i)}\right) \overline{\epsilon}_t+\alpha(\|\nabla f\|+\overline{\epsilon}_t) \epsilon_t^{(i)}+\alpha^2 L\left(\norm{\nabla f_i}+\epsilon_t^{(i)}\right) \epsilon_t^{(i)}+\alpha \overline{\epsilon}_t \epsilon_t^{(i)} +\frac{\alpha^2 L}{2}\left(\epsilon_t^{(i)}\right)^2 \\
&=\alpha\norm{\nabla f_i} \overline{\epsilon}_t+\alpha\|\nabla f\| \epsilon_t^{(i)}+\alpha^2 L\norm{\nabla f_i} \epsilon_t^{(i)}+3 \alpha \overline{\epsilon}_t \epsilon_t^{(i)}+\frac{3 \alpha^2 L}{2}\left(\epsilon_t^{(i)}\right)^2.
\end{aligned} 
\end{equation}

Now the question is transformed to bounding $\epsilon_t^{(i)}$ and $\overline{\epsilon}_t$. In the following, we define $\E_{\xi_t}[\cdot ]$ that represents taking expectation with respect to $\xi_0, \xi_1,...\xi_{t-1}$.
By the random exploration nature of our SGQ algorithm, we may write
\begin{equation}
\label{eq: epsilon_i}
\begin{aligned}
\E_{\xi_t}&\left[\left(\epsilon_t^{(i)}\right)^2\right]= \E_{\xi_t} \left[L_i^2\norm{x_{\tau_i^t}-x_t}^2\right]  \\
& \le L_i^2 \left(\frac{p}{n} \norm{x_{t-1}-x_t}^2+\frac{p}{n}\left(1-\frac{p}{n}\right)\norm{x_{t-2}-x_t}^2+\cdots \cdot \frac{p}{n}\left(1-\frac{p}{n}\right)^{t-2}\norm{x_1-x_t}^2 + \left(1-\frac{p}{n} \right)^{t-1} \norm{x_0 - x_t}^2 \right)\\
& = L_i^2 \cdot \frac{p}{n} \cdot \sum_{k=1}^{t-1}\left(1-\frac{p}{n}\right)^{k-1}\norm{x_{t-k}-x_t}^2+L_i^2\left(1-\frac{p}{n}\right)^{t-1} \cdot \norm{x_0-x_t}^2.
\end{aligned}
\end{equation}

Then it follows directly from Corollary \ref{coro: x_(t-k) - x_t} that
$$
\norm{x_{t-k}-x_t}=: B_k \le \frac{\alpha}{1-\alpha L_{\max }}\left(\frac{1}{1-\alpha L_{\max }}\right)^{k-1} \cdot \sum_{l=1}^k\norm{\nabla f_{i_{t-l}}(x_t)}.
$$

After plug the above bound for $\norm{x_{t-k} - x_t}$ back to \eqref{eq: epsilon_i}, we obtain
$$
\begin{aligned}
\E_{\xi_t}\left[\left(\epsilon_t^{(i)}\right)^2\right] &\le L_i^2 \cdot \frac{p}{n} \cdot \sum_{k=1}^{t-1}\left(\frac{\alpha}{1-\alpha L_{\max}}\right)^2\left(\frac{1-\frac{p}{n}}{\left(1-\alpha L_{\text {max }}\right)^2}\right)^{k-1} \left(\sum_{i=1}^k\norm{\nabla f_{i_{t-l}} (x_t)}\right)^2+L_i^2\left(1-\frac{p}{n}\right)^{t-1} \norm{x_0-x_t}^2 \\
& =\frac{\alpha^2 p L_i^2}{n \left(1-\alpha L_{\max}\right)^2} \sum_{k=1}^{t-1} q^{k-1}\left(\sum_{l=1}^k\norm{\nabla f_{i_{t-l}}(x_t)} \right)^2+ \frac{\alpha^2 L_i^2 \cdot q^{t-1}}{(1-\alpha L_{\max})^2} \left(\sum_{l=1}^k\norm{\nabla f_{i_{t-l}}(x_t)} \right)^2\\
\text { (Canchy-Schwartz) } &\le \frac{\alpha^2 p L_i^2}{n \left(1-\alpha L_{\text {max }}\right)^2} \sum_{k=1}^{t-1} q^{k-1} \cdot k \sum_{l=1}^k\norm{\nabla f_{i_{t-l}}(x_t)}^2 + \frac{\alpha^2 L_i^2 \cdot q^{t-1} t}{(1-\alpha L_{\max})^2} \sum_{l=1}^t\norm{\nabla f_{i_{t-l}}(x_t)}^2 \\
&= \frac{\alpha^2 p L_i^2}{n \left(1-\alpha L_{\max}\right)^2} \sum_{k=1}^{t-1} \sum_{l=1}^k k q^{k-1}  \norm{\nabla f_{i_{t-l}}(x_t)}^2 + \frac{\alpha^2 L_i^2 \cdot q^{t-1} t}{(1-\alpha L_{\max})^2} \sum_{l=1}^t\norm{\nabla f_{i_{t-l}}(x_t)}^2 \\
&= \frac{\alpha^2 p L_i^2}{n \left(1-\alpha L_{\max}\right)^2} \sum_{l=1}^t T_l  \norm{\nabla f_{i_{t-l}}(x_t)}^2 + \frac{\alpha^2 L_i^2 \cdot q^{t-1} t}{(1-\alpha L_{\max})^2} \sum_{l=1}^t\norm{\nabla f_{i_{t-l}}(x_t)}^2,
\end{aligned}
$$
where $q = \frac{1-\frac{p}{n}}{(1-\alpha L_{\max})^2}$. By using the stepsize assumption in Theorem \ref{thm: SGQ}, we have 
$$ q = \frac{1-\frac{p}{n}}{(1-\alpha L_{\max})^2} \le \frac{1-\frac{p}{n}}{(1- \frac{1 - \sqrt{1 - \frac{p}{2n}} }{L_{\max}}  L_{\max})^2} = \frac{1 - \frac{p}{n}}{1-\frac{p}{2n}} = 1 - \frac{p}{2n - p} < 1 -\frac{p}{2n} < 1,$$
and thus $T_l$ is validly defined as:
$$T_l = lq^{l-1} + ... +tq^{t-1} \le t \frac{q^{l-1} - q^t}{1-q} \le \frac{t}{1-q} q^{l-1}.$$

Hence we further have, 
$$
\begin{aligned}
\E_{\xi_t}\left[\left(\epsilon_t^{(i)}\right)^2\right] &\le \frac{\alpha^2 p L_i^2}{n \left(1-\alpha L_{\max}\right)^2}  \sum_{l=1}^{t}  \frac{t}{1-q} q^{l-1} \norm{\nabla f_{i_{t-l}}(x_t)}^2 + \frac{\alpha^2 L_i^2 \cdot q^{t-1} t}{(1-\alpha L_{\max})^2} \sum_{l=1}^t\norm{\nabla f_{i_{t-l}}(x_t)}^2\\
&\le \frac{\alpha^2 p L_i^2}{n \left(1-\alpha L_{\max}\right)^2}  \frac{t}{1-q} \sum_{l=1}^{t} q^{t-1} \norm{\nabla f_{i_{t-l}}(x_t)}^2 + \frac{\alpha^2 L_i^2 t}{(1-\alpha L_{\max})^2} \sum_{l=1}^t q^{l-1} \norm{\nabla f_{i_{t-l}}(x_t)}^2\\
& \le \left(\frac{p}{n(1-q)} + 1 \right) \frac{\alpha^2L_i^2 t}{(1-\alpha L_{\max})^2} \sum_{l=1}^t q^{l-1}  \norm{\nabla f_{i_{t-l}}(x_t)}^2.
\end{aligned}
$$

Similarly we can also provide a bound for $\E \left[\epsilon_t^{(i)} \overline{\epsilon}_t \right], \E \left[\epsilon_t^{(i)}\right]$, and $ \E \left[ \overline{\epsilon}_t\right]$. For $\E \left[\epsilon_t^{(i)} \overline{\epsilon}_t \right]$, we can write
$$
\begin{aligned}
\E_{\xi_t} \left[\epsilon_t^{(i)} \overline{\epsilon}_t \right] &= \frac{1}{n} \sum_{j=1}^n \E \left[\epsilon_t^{(i)}\epsilon_t^{(j)} \right] \le \frac{1}{n} \sum_{j=1}^n \sqrt{\E \left[(\epsilon_t^{(i)})^2 \right] \E \left[(\epsilon_t^{(j)})^2 \right]} \\
&= \left(\frac{p}{n(1-q)} + 1 \right) \frac{\alpha^2L_i L t}{(1-\alpha L_{\max})^2} \sum_{l=1}^t q^{l-1}  \norm{\nabla f_{i_{t-l}}(x_t)}^2.
\end{aligned}
$$

The calculation of $\E_{\xi_t} \left[\epsilon_t^{(i)} \right]$ and $\E_{\xi_t} \left[ \overline{\epsilon}_t \right]$ is similar but simpler. In fact,

$$ 
\E_{\xi_t} \left[\epsilon_t^{(i)}\right] = \left(\frac{p}{n(1-\tilde{q})} + 1 \right) \frac{\alpha L_i}{1-\alpha L_{\max}} \sum_{l=1}^t 
\tilde{q}^{l-1}  \norm{\nabla f_{i_{t-l}}(x_t)}
$$

$$ 
\E_{\xi_t} \left[ \overline{\epsilon}_t\right] = \frac{1}{n}\sum_{j=1}^n \E \left[ \epsilon_t^{(j)}  \right] = \left(\frac{p}{n(1-\tilde{q})} + 1 \right) \frac{\alpha L}{1-\alpha L_{\max}} \sum_{l=1}^t \tilde{q}^{l-1}  \norm{\nabla f_{i_{t-l}}(x_t)},
$$
where $\tilde{q} = \frac{1 -p/n}{1 -\alpha L_{\max}} \le q < 1$. 

We can now take expectation with respect to $x_t$ on equation \eqref{eq: r_i estimate}, and then plug those estimations on $\epsilon_t^{(i)}, \overline{\epsilon}_t$ back. We obtain that
$$
\begin{aligned}
\E_{\xi} [r_i^t] &\le \left( \frac{p}{n(1-\tilde{q})} + 1 \right) 
\frac{\alpha^2(L + L_i +\alpha L L_i)}{2(1 - \alpha L_{\max})} \sum_{l=1}^{t} \tilde{q}^{l-1} \norm{\nabla f_{i_{t-l}}(x_t)}^2  \\
& + \left( \frac{p}{n(1-q)} + 1 \right) 
\frac{3\alpha^3 L L_i (2 + \alpha L_i)t}{4(1 - \alpha L_{\max})^2} \sum_{l=1}^{t} q^{l-1} \norm{\nabla f_{i_{t-l}}(x_t)}^2 + \left( \frac{p}{n(1-\tilde{q})} + 1 \right) 
\frac{\alpha^2 L_i}{2(1 - \alpha L_{\max})} \sum_{l=1}^{t} \tilde{q}^{l-1} \norm{\nabla f(x_t)}^2 \\
& + \left( \frac{p}{n(1-\tilde{q})} + 1 \right) 
\frac{\alpha^2(L +\alpha L L_i)}{2(1 - \alpha L_{\max})} \sum_{l=1}^{t} \tilde{q}^{l-1} \norm{\nabla f_i(x_t)}^2 + 
\left( \frac{p}{n(1-q)} + 1 \right) 
\frac{3\alpha^3 L L_i (2 + \alpha L_i)t}{4(1 - \alpha L_{\max})^2} \sum_{l=1}^{t} q^{l-1} \norm{\nabla f_i(x_t)}^2.
\end{aligned}$$ 

Since we have $\tilde{q} \le q$ and Assumption \ref{assumption: Delta bound} holds. For any index $i$, we are allowed to bound
$$
\begin{aligned}
\sum_{l=1}^{t} \tilde{q}^{l-1}  \norm{\nabla f_{i_{t-l}}(x_t)}^2 &\le \sum_{l=1}^{t} q^{l-1}  \norm{\nabla f_{i_{t-l}}(x_t)}^2 = \sum_{l=1}^{t} q^{l-1}  \norm{\nabla f_{i_{t-l}}(x_t) - \nabla f(x_t) + \nabla f(x_t)}^2 \\
&\le 2\sum_{l=1}^{t} q^{l-1}  \left (\norm{\nabla f_{i_{t-l}}(x_t) - \nabla f(x_t)}^2 + \norm{\nabla f(x_t)}^2 \right) \\
\text{[By Assumption \ref{assumption: Delta bound}]}\quad &\le 2\sum_{l=1}^{t} q^{l-1}  \left (\Delta^2 + \norm{\nabla f(x_t)}^2 \right)\\
&\le \frac{2}{1 - q} \left (\Delta^2 + \norm{\nabla f(x_t)}^2 \right).
\end{aligned} 
$$

Then we can use the above inequality to relax the bound for $\E [r_i^t]$, in fact we have
$$
\begin{aligned}
\E_{\xi_t} [r_i^t] &\le \alpha^2 \left (\frac{p}{n(1-q)} + 1 \right ) \frac{L + L_{\max}}{1-\alpha L_{\max}} \frac{2}{1-q} \cdot (\Delta^2 + \norm{\nabla f(x_t)}^2 ) + \mathcal{O}(\alpha^3) \\
\text{[$\alpha < \frac{1}{2L_{\max}}$]} \quad &\le \alpha^2 \left (\frac{p}{n(1-q)} + 1 \right ) \frac{4(L + L_{\max})}{1-q} \cdot (\Delta^2 + \norm{\nabla f(x_t)}^2 ) + \mathcal{O}(\alpha^3) \\
\text{[$ q< 1 -\frac{p}{2n} $]} \quad & \le \alpha^2 \cdot \frac{24n(L + L_{\max})}{p} \cdot (\Delta^2 + \norm{\nabla f(x_t)}^2) + \mathcal{O}(\alpha^3).
\end{aligned}
$$

By replacing $\E_{\xi_t}[r_i^t]$ in the per-step analysis, we obtain that
$$
\begin{aligned}
\E[f(x_t) - f(x_{t+1})| x_t] &\ge \left(1 + (1-p)\sqrt{\frac{C_1}{2c}} \right ) \alpha \norm{\nabla f(x_t)}^2 - \frac{\alpha^2 L}{2} \left( 1 - (1-p)\sqrt{\frac{2C_2 }{c}} \right ) \E_i\left[\norm{\nabla f_i(x_t)}^2\right] + 2(1-p)r^t_{i_t} \\
& = \bigg[\left( 1 + (1-p)\sqrt{ \frac{C_1}{2 c}} \right)\alpha + \alpha^2 (1 - p) \frac{24n(L + L_{\max})}{p} \bigg] \cdot \norm{\nabla f(x_t)}^2 \\
& - \frac{\alpha^2 L}{2} \left( 1 - (1-p)\sqrt{\frac{2C_2 }{c}} \right ) \E_i\left[\norm{\nabla f_i(x_t)}^2\right] + \alpha^2 (1 - p) \frac{24n(L + L_{\max})}{p} \cdot \Delta^2 + \mathcal{O}((1-p)\alpha^3)
\end{aligned}
$$

Thus, by applying the variance transfer lemma \ref{lemma: variance_transfer} and Assumption \ref{assumption: f PL-condition}, we have
$$
\begin{aligned}
\E[f(x_{t+1})] - \inf f \leq \bigg[ 1 - &\left( 1 + (1-p)\sqrt{ \frac{C_1}{2 c}} \right) \alpha 2 \mu + 2 \alpha^2 L L_{\max} \left( 1 - (1 - p)\sqrt{\frac{2 C_2}{c}} \right) + \\ 
& \alpha^2 (1 - p) \frac{24n(L + L_{\max})}{p} \cdot 2\mu \bigg] \cdot (\E[f(x_t) ]- \inf f)  + \alpha^2 L \left( 1 - (1-p)\sqrt{\frac{2 C_2}{c}} \right) \sigma_f^* \\
& + \alpha^2 (1 - p) \frac{24n(L + L_{\max})}{p} \cdot \Delta^2 + \mathcal{O}((1-p)\alpha^3).
\end{aligned}
$$

Then by using the stepsize assumption in Theorem \ref{thm: SGQ}, we are able to simplify the coefficient in front of $(\E[f(x_t) ]- \inf f)$:
$$ 
\begin{aligned}
&1 - \left( 1 + (1-p)\sqrt{ \frac{C_1}{2 c}} \right) \alpha 2 \mu + 2 \alpha^2 L L_{\max} \left( 1 - (1 - p)\sqrt{\frac{2 C_2}{c}} \right ) +
\alpha^2 (1 - p) \frac{24n(L + L_{\max})}{p} \cdot 2\mu \\
& = 1 -  \Bigg [\left( 2 + 2(1-p)\sqrt{ \frac{C_1}{2 c}} \right) + 2 \frac{\alpha L L_{\max}}{\mu} \left( 1 - (1 - p)\sqrt{\frac{2 C_2}{c}} \right )+ \alpha (1 - p) \frac{48n(L + L_{\max})}{p} \Bigg]\alpha \mu.
\end{aligned}
$$

From the stepsize assumption, we know that 
$$ \alpha \le \frac{\mu}{4L L_{\max}}, \quad \alpha \le \frac{p}{96n(L +L_{\max})}\cdot \frac{1}{1-p}.$$

Then the the coefficient in front of $(\E[f(x_t) ]- \inf f)$ has the following reduction:
$$ 
\begin{aligned}
& 1 -  \Bigg [\left( 2 + 2(1-p)\sqrt{ \frac{C_1}{2 c}} \right) + 2 \frac{\alpha L L_{\max}}{\mu} \left( 1 - (1 - p)\sqrt{\frac{2 C_2}{c}} \right )+ \alpha (1 - p) \frac{48n(L + L_{\max})}{p} \Bigg]\alpha \mu \\
& \le 1 -  \Bigg [\left( 2 + 2(1-p)\sqrt{ \frac{C_1}{2 c}} \right) - \frac{1}{2}\left( 1 - (1 - p)\sqrt{\frac{2 C_2}{c}} \right )- \frac{1}{2} \Bigg]\alpha \mu \\
& = 1 -  \Bigg [ 1 + 2(1-p)\sqrt{ \frac{C_1}{2 c}} + \frac{1}{2}(1 - p)\sqrt{\frac{2 C_2}{c}} \Bigg]\alpha \mu \\
& = 1 -  \left[ 1 + \frac{(1-p)(2\sqrt{C_1} + \sqrt{C_2})}{\sqrt{2c}} \right]\alpha \mu.
\end{aligned}
$$

Now we have,
$$
\begin{aligned}
\E[f(x_{t+1})] - \inf f \leq& \left(1 -   \left( 1 + \frac{(1-p)(2\sqrt{C_1} + \sqrt{C_2})}{\sqrt{2c}} \right)\alpha \mu \right )(\E[f(x_t) ]- \inf f)  + \alpha^2 L \left( 1 - (1-p)\sqrt{\frac{2 C_2}{c}} \right) \sigma_f^* \\
& + \alpha^2 (1 - p) \frac{24n(L + L_{\max})}{p} \cdot \Delta^2 + \mathcal{O}((1-p)\alpha^3).
\end{aligned}
$$

Recursively using the above inequality and telescoping them together, we obtain that
$$
\begin{aligned}
\E[f(x_t)] - \inf f \leq& \left(1 -   \left( 1 + \frac{(1-p)(2\sqrt{C_1} + \sqrt{C_2})}{\sqrt{2c}} \right)\alpha \mu \right )^t(f(x_0) - \inf f)  + \frac{\alpha L}{\mu} \frac{\sqrt{2c} - 2(1-p)\sqrt{C_2}}{\sqrt{2c} +(1-p)(2\sqrt{C_1} + \sqrt{C_2})} \sigma_f^* \\
& + \frac{\alpha \cdot 24n(L + L_{\max})}{p\mu } \frac{(1 - p) \sqrt{2c} }{\sqrt{2c} +(1-p)(2\sqrt{C_1} + \sqrt{C_2})}\cdot \Delta^2 + \mathcal{O}((1-p)\alpha^2),
\end{aligned}
$$
where the constant hidden in $\mathcal{O}(\cdot)$ is a rational polynomial that depends on $\sqrt{C_1}, \sqrt{C_2}, \mu, \sqrt{c}, \Delta, L, L_{\max}$.

\end{proof}

\noindent \textbf{Proof of Proposition \ref{prop: C_1, C_2}}
\begin{proof}
We first expand the variance of $\text{EI}_i$, and express the formula with $ \bar{g}(x) $, $ V_g(x) $, $ S_g(x) $, and $ \kappa_g(x)$. We denote the empirical standard deviation by $\sigma_g(x)$. For simplicity of notation, we will omit the dependence on $x$ when it is clear from the context.
$$ 
\begin{aligned}
  \mathrm{Var}\Bigg[ \Bigg\{ &\alpha f'(x)f_i(x) - \frac{\alpha^2L}{2} f_i'(x)^2  \Bigg \}_{i=1}^n \Bigg] \\
  &= \alpha^2 \left(  \frac{(1+\alpha L) (f')^2}{n} \sum_{i=1}^n(f_i')^2 - (f')^4 - \frac{ \alpha L f'}{n}\sum_{i=1}^n f_i^3 + \frac{\alpha^2 L^2}{4n} \Big(\sum_{i=1}^n f_i'^4 - \frac{1}{n}\big(\sum_{i=1}^n f_i'^2 \big)^2 \Big)  \right) \\
  & = \alpha^2 (1 + \alpha L) \bar{g}^2 (\bar{g}^2 + \sigma_g^2)- \alpha^2 \bar{g}^4 - \alpha^3 L \bar{g} (S_g\sigma_g^3 + 3\bar{g} \sigma_g^2 + \bar{g}^3)   + \frac{\alpha^4 L^2}{4} \left( (\kappa_g-1) \sigma_g^4 + 4 \bar{g} S_g \sigma_g^3 + 4 \bar{g}^2 \sigma_g^2 \right) \\
&= \alpha^2 \left[ (1 - \alpha L)^2 \bar{g}^2 \sigma_g^2 - \alpha L (1 - \alpha L ) \bar{g} S_g \sigma_g^3 + \frac{\alpha^2 L^2}{4} (\kappa_g-1) \sigma_g^4 \right].
\end{aligned}
$$

Then when $\alpha L \le 1$, we have

$$ 
\begin{aligned}\alpha L (1 - \alpha L ) |\bar{g}| |S_g| \sigma_g^3 &\le  \alpha L (1 - \alpha L ) |\bar{g}| \cdot (1 - \delta) \sqrt{\kappa_g - 1} \sigma_g^3 \\
& = 2 \left ( \sqrt{1 - \delta} \cdot (1 - \alpha L) |\bar{g}| \sigma_g \right ) \cdot \left ( \sqrt{1 - \delta} \cdot \frac{\alpha L}{2} \sqrt{\kappa_g - 1} \cdot  \sigma_g^2 \right ) \\
\text{ [Basic inequality: $2ab \le a^2 + b^2$] }  \quad &\le (1 - \delta) (1 - \alpha L)^2 \bar{g}^2 \sigma_g^2 +  (1 - \delta)\frac{\alpha^2 L^2}{4} (\kappa_g-1) \sigma_g^4 \\
 & = (1- \delta) \left[ (1 - \alpha L)^2 \bar{g}^2 \sigma_g^2 + \frac{\alpha^2 L^2}{4} (\kappa_g-1) \sigma_g^4 \right].
\end{aligned}
$$

Therefore, the variance of $\text{EI}_i$ can be written as
$$ 
\begin{aligned}
  \mathrm{Var}\Bigg[ \Bigg\{ \alpha f'(x)f_i(x) &- \frac{\alpha^2L}{2} f_i'(x)^2  \Bigg \}_{i=1}^n  \Bigg] = \alpha^2 \left[ (1 - \alpha L)^2 \bar{g}^2 \sigma_g^2 - \alpha L (1 - \alpha L ) \bar{g} S_g \sigma_g^3 + \frac{\alpha^2 L^2}{4} (\kappa_g-1) \sigma_g^4 \right],\\
  &\ge \alpha^2 \left[ (1 - \alpha L)^2 \bar{g}^2 \sigma_g^2 - (1- \delta) \left[ (1 - \alpha L)^2 \bar{g}^2 \sigma_g^2 + \frac{\alpha^2 L^2}{4} (\kappa_g-1) \sigma_g^4 \right] + \frac{\alpha^2 L^2}{4} (\kappa_g-1) \sigma_g^4 \right] \\
  &= \delta \alpha^2 (1-\alpha L)^2 \bar{g}^2 \sigma_g^2 + \delta \frac{\alpha^4 L^2}{4} (\kappa_g-1) \sigma_g^4.
\end{aligned}
$$

If we pick $\alpha$ such that $\alpha L \le \frac{1}{2}$,  then
$$ 
\begin{aligned}
  \mathrm{Var}\Bigg[ \Bigg\{ \alpha f'(x)f_i(x) - \frac{\alpha^2L}{2} f_i'(x)^2  \Bigg \}_{i=1}^n \Bigg]  &= \delta \alpha^2 (1-\alpha L)^2 \bar{g}^2 \sigma_g^2 + \delta \frac{\alpha^4 L^2}{4} (\kappa_g-1) \sigma_g^4\\
  \text{ [$\alpha L \le \frac{1}{2}$]} \quad &\ge  \alpha^2 \cdot \frac{\delta V_g}{4 \bar{g}^2} \bar{g}^4 + \delta \frac{\alpha^4 L^2}{4} (\kappa_g-1) V_g^2 \\
  &= \frac{\delta V_g}{4 \bar{g}^2} \cdot \alpha^2 \bar{g}^4 + \frac{\delta (\kappa_g-1)V_g^2}{4(V_g + \bar{g}^2)^2} \cdot \alpha^4 L^2 \left( \sum_{i=1}^n \norm{\nabla f_i(x)}^2 \right)^2.
\end{aligned}
$$

Then we may set
$$
C_1 = \inf_{x\in \R}\frac{\delta(x) V_g(x)}{4\bar{g}^2(x)}, \inf{x \in \R} \quad C_2 = \inf_{x\in \R}\frac{\delta(x) (\kappa_g(x) - 1)V_g(x)^2}{4\left(V_g(x) + \bar{g}^2(x) \right)^2}.
$$ 

This finishes the proof.

\end{proof}

\noindent \textbf{Proof of Proposition \ref{prop: tilde_c}}
\begin{proof} 
The first argument of $\tilde{c}(x) \le n-1$ is already proved in Lemma \ref{lemma: EI difference}. We first start with proving that if $\D_x$ is sub-gaussian, then with probability $1 - 1/n$, we have $\tilde{c}(x) \le \mathcal{O}(\sigma_{\D_x}\sqrt{\log n})$.

This can be proved by using the union bound and sub-gaussian property. Assume that the distribution $\D_x$ has mean $\mu_{\D_x}$ and variance proxy $\sigma_{\D_x}^2$, then

$$
\begin{aligned}
\Pb \left( \min_i EI_i < \mu_{\D_x} - t \right) &\le \sum_{i=1}^n \Pb(EI_i \le \mu_{\D_x} - t) \\
& \le n \exp{\left(-\frac{t^2}{2\sigma_{\D_x}^2} \right)},
\end{aligned}$$
where the first inequality uses the union bound and the second inequality uses the fact that $\D_x$ is sub-gaussian. Let $t =  2\sigma_{\D_x} \sqrt{\log 3n}$, we obtain

\begin{equation}
\label{eq: min EI_i estimate}
 \Pb \left( \min_i EI_i < \mu_{\D_x} - 2 \sigma_{\D_x} \sqrt{\log 3n}  \right) \le n\frac{1}{(3n)^2} < \frac{1}{3n}.
\end{equation}

Thus $\min_i EI_i \ge \mu_{\D_x} - 2 \sigma_{\D_x} \sqrt{\log 3n}$ holds with probability at least $1 - \frac{1}{3n}$. 

Then by sub-gaussian property, we can control how much the sample mean deviates from its expectation:
$$\Pb \left(\frac{1}{n}\sum_{i=1}^n EI_i \geq \mu_{\D_x} + t \right) \leq \exp\left( -\frac{n t^2}{2\sigma^2_{\D_x}} \right).$$

Plugging $t = \sigma_{\D_x}\sqrt{\frac{2}{n} \log 3n}$, we obtain 
\begin{equation}
\label{eq: EI_bar estimate}
    \Pb \left(\frac{1}{n}\sum_{i=1}^n EI_i \geq \mu_{\D_x} + \sigma_{\D_x} \sqrt{\frac{2}{n} \log 3n} \right) \leq \frac{1}{3n}.
\end{equation}

The last thing is to argue that there exist a constant $\epsilon_{\D_x}$ such that $\max_i EI_i - \mu_{\D_x} \ge \epsilon_{\D_x}$ holds with non-zero probability. This is true since we assume that $\D_x$ is a non-degenerate distribution. In specific, there always exists a positive threshold $\epsilon_{\D_x}$, such that 
$$ \Pb \left(EI_i - \mu_{\D_x} \ge \epsilon_{\D_x} \right) \ge \frac{1}{3},$$
for each $i$. Then, we will have
\begin{equation}
\label{eq: max EI_i estimate}
\Pb \left( \max_i EI_i - \mu_{\D_x} \ge \epsilon_{\D_x} \right) \ge 1 - (1 -\frac{1}{3})^n \ge 1 - \frac{1}{3n},
\end{equation}
when $n$ sufficiently large. Now combine \eqref{eq: min EI_i estimate},\eqref{eq: EI_bar estimate}, and apply the union bound, we have
$$
\begin{aligned}
\frac{\frac{1}{n}\sum_{i=1}^n EI_i - \min_i EI_i}{\max_i EI_i - \frac{1}{n}\sum_{i=1}^n EI_i} &\le \frac{\mu_{\D_x} + \sigma_{\D_x}\sqrt{\frac{2}{n} \log 3n} - \left( \mu_{\D_x} - 2 \sigma_{\D_x} \sqrt{\log 3n} \right)}{\max_i EI_i - \left(\mu_{\D_x} + \sigma_{\D_x}\sqrt{\frac{2}{n} \log 3n}\right )} \\
& = \frac{ \sigma_{\D_x}\sqrt{\frac{2}{n} \log 3n} + 2 \sigma_{\D_x} \sqrt{\log 3n} }{\max_i EI_i - \left(\mu_{\D_x} + \sigma_{\D_x}\sqrt{\frac{2}{n} \log 3n} \right )},
\end{aligned}
$$
holds with probability $1 - \frac{2}{3n}$. Since when, $n$ sufficiently large, we will have 
$$\sigma_{\D_x}\sqrt{\frac{2}{n} \log 3n} \le \frac{\epsilon_{\D_x}}{2}.$$

Thus by \eqref{eq: max EI_i estimate}
$$\max_i EI_i - \left(\mu_{\D_x} + \sigma_{\D_x}\sqrt{\frac{2}{n} \log 3n} \right ) \ge \mu_{\D_x} + \epsilon_{\D_x} - \left(\mu_{\D_x} + \sigma_{\D_x}\sqrt{\frac{2}{n} \log 3n} \right ) \ge \frac{\epsilon_{\D_x}}{2},$$
holds with probability $1 - \frac{1}{3n}$, when $n$ sufficiently large. Finally, by utilizing the union bound again, we obtain that
$$
\begin{aligned}
\frac{\frac{1}{n}\sum_{i=1}^n EI_i - \min_i EI_i}{\max_i EI_i - \frac{1}{n}\sum_{i=1}^n EI_i} & \le \frac{\sigma_{\D_x}\sqrt{\frac{8}{n} \log 3n} + 4 \sigma_{\D_x} \sqrt{\log 3n} }{\epsilon_{\D_x}},
\end{aligned}
$$
holds with probability $1 - \frac{1}{n}$, when $n$ sufficiently large.
Define 
$$\tilde{c}(x) := \frac{\sigma_{\D_x}\sqrt{\frac{8}{n} \log 3n} + 4 \sigma_{\D_x} \sqrt{\log 3n} }{\epsilon_{\D_x}} = \mathcal{O}(\sigma_{\D_x}\sqrt{\log n}).$$

This finishes proving the existence of $\tilde{c}(x)$.

Second, we prove that if $\D_x$ is bounded, and denote the bound of $\D_x$ by $B_{\D_x}$ (i.e. $|EI_i| \le B_{\D_x}$), then there exist $\tilde{c}(x) = \mathcal{O}(B_{\D_x})$ with probability $1 - \frac{1}{n}$. Due to the boundness of $EI_i$, we have
$$
\frac{\frac{1}{n}\sum_{i=1}^n EI_i - \min_i EI_i}{\max_i EI_i - \frac{1}{n}\sum_{i=1}^n EI_i} \le \frac{\frac{1}{n}\sum_{i=1}^n EI_i + B_{\D_x}}{\max_i EI_i - \frac{1}{n}\sum_{i=1}^n EI_i}.
$$

Then by Hoeffding's Inequality, we can bound how much the sample mean deviates from its expectation similarly to \eqref{eq: EI_bar estimate}. In specific, since $EI_i$ are bounded, we have
$$
\Pb \left(\frac{1}{n}\sum_{i=1}^n EI_i \geq \mu_{\D_x} + t \right) \leq \exp\left( -\frac{2n t^2}{4 B_{\D_x}^2} \right).
$$

If we let $t = B_{\D_x}\sqrt{\frac{2}{n}\log 2n}$, then
$$
\Pb \left(\frac{1}{n}\sum_{i=1}^n EI_i \geq \mu_{\D_x} + B_{\D_x}\sqrt{\frac{2}{n}\log 2n} \right) \leq \frac{1}{2n}.
$$

Then with the similar analysis in equation \eqref{eq: max EI_i estimate}, we know that there exists a constant $\epsilon_{\D_x}$ such that 
$$
\Pb \left( \max_i EI_i - \mu_{\D_x} \ge \epsilon_{\D_x} \right) \ge 1 - \frac{1}{2n}.$$

Also, since when $n$ sufficiently large, we will have 
$$B_{\D_x}\sqrt{\frac{2}{n} \log 2n} \le \frac{\epsilon_{\D_x}}{2}.$$

Thus we obtain that
$$\max_i EI_i - \left(\mu_{\D_x} + B_{\D_x}\sqrt{\frac{2}{n} \log 2n} \right ) \ge \mu_{\D_x} + \epsilon_{\D_x} - \left(\mu_{\D_x} + B_{\D_x}\sqrt{\frac{2}{n} \log 2n} \right ) \ge \frac{\epsilon_{\D_x}}{2},$$
holds with probability $1 - \frac{1}{2n}$, for sufficiently large $n$. Then by integrating the above probability bound together, we obtain that
$$
\frac{\frac{1}{n}\sum_{i=1}^n EI_i - \min_i EI_i}{\max_i EI_i - \frac{1}{n}\sum_{i=1}^n EI_i} \le \frac{\mu_{\D_x} + B_{\D_x}\sqrt{\frac{2}{n}\log 2n} + B_{\D_x}}{\epsilon_D/2} \le \frac{2B_{\D_x} + B_{\D_x}\sqrt{\frac{2}{n}\log 2n}}{\epsilon_D/2} 
$$
holds with probability $1 - \frac{1}{n}$, for large $n$. Define $ \tilde{c}(x) := \frac{2B_{\D_x} + B_{\D_x}\sqrt{\frac{2}{n}\log 2n}}{\epsilon_D/2}  = \mathcal{O}(B_{\D_x})$, this finishes the proof.

\end{proof}

\newpage
\subsection{Supporting Results}
\noindent \textbf{Proof of the Lower Bound for $\sqrt{c/2} - \sqrt{C_2}$:} 
\begin{proof}We claim that $c \ge 4C_2$. The reason is as follows:

By the definition of $C_2$ in Assumption \ref{assumption: Hete}, the inequality  
$$ \mathrm{Var}\Big[ \Big \{\alpha \langle \nabla f(x), \nabla f_i(x)\rangle - \frac{\alpha^2L}{2} \norm{\nabla f_i(x)}^2 \Big\}_{i=1}^n \Big] \ge C_1 \alpha^2 \norm{\nabla f(x)}^4 + C_2 \alpha^4 L^2 \cdot \E_i \left [\norm{\nabla f_i(x)}^2 \right ]^2,
$$
needs to be satisfied for all $x \in\R^d$. Specifically, we let $x = x^*$ be the point that minimizes $f(x)$. Then by noticing $\nabla f(x^*) = 0$, we have 
$$
\mathrm{Var}\left[ \Big \{\alpha \langle \nabla f(x^*), \nabla f_i(x^*)\rangle -\frac{\alpha^2L}{2} \norm{\nabla f_i(x^*)}^2 \Big \}_{i=1}^n \right] = \frac{\alpha^4L^2}{4} \E_i \left[ \norm{\nabla f_i(x^*)}^4 \right] \ge C_2 \alpha^4 L^2 \cdot \E_i \left [\norm{\nabla f_i(x^*)}^2 \right ]^2,
$$
which gives
$$C_2 \le \frac{1}{4}\frac{\E_i \left[ \norm{\nabla f_i(x^*)}^4 \right]}{\E_i \left [\norm{\nabla f_i(x^*)}^2 \right ]^2}.$$

And from the definition of $c$, we also have 
$$\max_i \text{EI}_i(x) - \E_i[\text{EI}_i(x)] \ge \sqrt{\frac{\mathrm{Var}[\{\text{EI}_i(x)\}_{i=1}^n]}{c}}, \quad \forall x \in \mathbb{R}^d.$$

We again let $x = x^*$ and obtain
$$ -\min_i \norm{\nabla f_i(x^*)}^2 +  \E_i[\norm{\nabla f_i(x^*)}^2] \ge \sqrt{\frac{\E_i \left[ \norm{\nabla f_i(x^*)}^4 \right]}{c}}, \quad \forall x \in \mathbb{R}^d.$$

Thus by squaring both sides, we have
$$c \ge \frac{\E_i \left[ \norm{\nabla f_i(x^*)}^4 \right]}{\left ( \E_i \left [\norm{\nabla f_i(x^*)}^2 \right ] - \min_i \norm{\nabla f_i(x^*)}^2\right)^2 } \ge \frac{\E_i \left[ \norm{\nabla f_i(x^*)}^4 \right]}{ \E_i \left [\norm{\nabla f_i(x^*)}^2 \right ]^2} \ge 4C_2.$$

This proves our claim.

\end{proof}

\begin{lemma}[Variance transfer]
\label{lemma: variance_transfer}
If Assumptions \ref{assumption: f_i convex} and \ref{assumption: f_i smooth} hold, then for all $x \in \mathbb{R}^d$, we have
$$
\E_i \left[ \norm{\nabla f_i(x)}^2 \right] \leq 4 L_{\max} (f(x) - \inf f) + 2\sigma_f^*.
$$
\end{lemma}

\begin{proof}
Let $x^* \in \arg\min f$, we start by writing
$$
\norm{\nabla f_i(x)}^2 \leq 2\norm{\nabla f_i(x) - \nabla f_i(x^*)}^2 + 2\norm{\nabla f_i(x^*)}^2.
$$

Taking the expectation with respect to $i$, we obtain that
$$
\E_i \left[ \norm{\nabla f_i(x)}^2 \right] \leq 2\E_i \left[ \norm{\nabla f_i(x) - \nabla f_i(x^*)}^2 \right] + 2\sigma_f^*.
$$

Now it remains to show that 
\begin{equation}
\label{eq: variance transfer}
2\E_i \left[ \norm{\nabla f_i(x) - \nabla f_i(x^*)}^2 \right] \le 4 L_{\max} (f(x) - \inf f).
\end{equation}

This is achieved by utilizing the coercivity of convex and smooth functions. In fact, we have 
$$\frac{1}{2L_{\max}} \norm{\nabla f_i(y) - \nabla f_i(x)}^2 \le \frac{1}{2L_i} \norm{\nabla f_i(y) - \nabla f_i(x)}^2 \le f_i(y) - f_i(x) - \langle \nabla f_i(x), y-x \rangle,
$$
for all $x,y \in \R^d$ and $i \in [[1,n]]$. Thus, by plugging $y=x, x=x^*$, we obtain that 
$$\frac{1}{2L_{\max}} \norm{\nabla f_i(x) - \nabla f_i(x^*)}^2 \le f_i(x) - f_i(x^*) - \langle \nabla f_i(x^*), x-x^* \rangle.
$$

If we take the expectation of the above equation and notice that $\E_i[\nabla f_i(x^*)] = \nabla f(x^*) = 0$, we will get
$$\frac{1}{2L_{\max}}\E_i \left[ \norm{\nabla f_i(x) - \nabla f_i(x^*)}^2 \right]\le f(x) - \inf f,
$$
which aligns with \eqref{eq: variance transfer}. This proves the variance transfer lemma.

\end{proof}

\end{document}